%% file: preprint.tex
\documentclass{article}
\usepackage{iclr2027_conference,times}

\input{math_commands.tex}

\usepackage{algorithm}\usepackage{algpseudocode}
\usepackage{tikz}\usetikzlibrary{positioning,arrows.meta,calc,decorations.pathreplacing}
\usepackage{placeins}
\usepackage{wrapfig}
\usepackage{enumitem}
\usepackage{booktabs}
\usepackage{tabularx}
\usepackage{amsthm}
\usepackage{amssymb}
\usepackage{url}
\usepackage{hyperref}
\usepackage{cleveref}

\newcommand{\iid}{\overset{\mathrm{iid}}{\sim}}
\newcommand{\rmseprior}{\mathrm{RMSE}_{\mathrm{prior}}}
\newcommand{\ratiostar}{\rho^{\star}}
\DeclareMathOperator{\tr}{tr}  
\newcommand{\lpipsmap}{\Psi_{\mathrm{LPIPS}}}

\newcommand{\rhostarpx}{\rho^{\star}_{\mathrm{px}}}
\newcommand{\rhostarlp}{\rho^{\star}_{\lpipsmap}}
\newcommand{\zT}{z^{\,\Psi}}
\newcommand{\zhatT}{\hat z^{\,\Psi}}

\newtheorem{theorem}{Theorem}[section]
\newtheorem{lemma}[theorem]{Lemma}
\newtheorem{proposition}[theorem]{Proposition}
\newtheorem{corollary}[theorem]{Corollary}
\newtheorem{definition}[theorem]{Definition}
\newtheorem{assumption}[theorem]{Assumption}
\newtheorem{remark}[theorem]{Remark}

\crefname{assumption}{assumption}{assumptions}
\Crefname{assumption}{Assumption}{Assumptions}

\title{From Pixels to Privacy-Sensitive Features:  \\ An Information-Theoretic Framework for \\ Reconstruction Bounds}

\author{Henry Bourne, Robert F. Allison \\
Department of Mathematics\\
The University of Bristol\\
Beacon House, Queens Road, Bristol, BS8 1QU, United Kingdom \\
\texttt{\{jd18380,marfa\}@bristol.ac.uk} 
}

\iclrfinalcopy
\begin{document}

\maketitle 
\lhead{Preprint. Under review.}
\renewcommand{\headrulewidth}{0pt}

\begin{abstract} 
What can an adversary recover from a model trained with DP-SGD? Its $(\varepsilon,\delta)$ guarantee does not answer this question directly. We derive a distortion bound that can be computed before training and applies even to an adversary who observes the parameter trajectory during training, knows every other training record and the population distribution from which the training records are drawn. The analyst may choose the representation in which error is measured, allowing the bound to target the aspects of reconstruction that are privacy-sensitive in a given application. Our information-theoretic framework separates two contributions: the information about a record revealed by training, measured through conditional mutual information, and the uncertainty already present in the data distribution, measured through source entropy. Rate--distortion theory combines them into a lower bound on the expected error of every reconstruction attack. For DP-SGD, the leakage bound depends only on the sampling rates and noise multipliers, while the source entropy is lower-bounded from a sample with finite-sample confidence. An orthogonal block construction keeps the resulting bound informative for high-dimensional images. We instantiate the framework in pixel space and in LPIPS feature space, obtaining, to our knowledge, the first lower bound on reconstruction distortion measured directly by LPIPS. On a range of image datasets the bounds establish non-trivial unavoidable expected reconstruction error within heuristically recommended privacy-budget ranges and beyond. Experiments compare the bounds with strong reconstruction attacks and identify the principal sources of the remaining slack. 
\end{abstract}

\section{Introduction}
\label{sec:intro}

\begin{figure}[b]
    \centering
    \includegraphics[width=\textwidth]{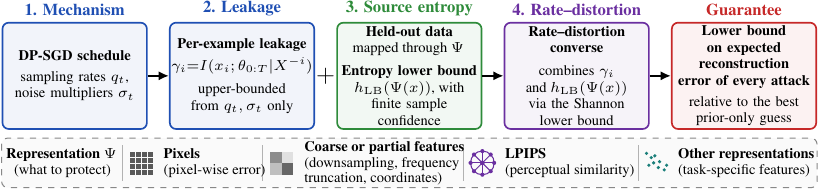}
    \caption{\textbf{Modular information-theoretic framework.} Mechanism leakage and a held-out entropy bound in a chosen representation $\Psi$ combine through rate--distortion into a lower bound on every attack; each module can be tightened or replaced independently.}
    \label{fig:framework}
    \vspace{-6pt}
\end{figure}

Models trained on sensitive data can reveal individual training examples through gradients, model parameters, or query access \citep{zhu2019deep,geiping2020inverting,haim2022reconstructing,Carlini2020ExtractingTD}. The standard defence is differentially private stochastic gradient descent (DP-SGD), which clips per-example gradients, adds Gaussian noise, and reports its guarantee through $(\varepsilon,\delta)$ \citep{abadi2016deep}. These parameters provide a general privacy guarantee, but they do not directly answer the operational question considered here: \emph{How accurately could an adversary reconstruct a training example?} This question is especially relevant when participation in the training set is already known but the content of an individual record remains sensitive, or when useful model accuracy requires a large privacy budget, $\varepsilon$. For orientation we use the tiers of \citet{ponomareva2023dp}: strong ($\varepsilon\leq1$), reasonable ($1<\varepsilon\leq10$) and weak ($\varepsilon>10$); the upper two are justified by common practice and empirical attacks rather than by a guarantee.

Reconstruction difficulty depends partly on what can be inferred from the training-data distribution before the trained model and its trajectory are observed. Even without seeing the training mechanism, an adversary who knows the training-data distribution can mount an effective guessing attack when that distribution has low entropy. We therefore measure reconstruction error relative to the best prior-only reconstruction. The resulting reconstruction ratio equals one when training reveals nothing useful beyond the training-data distribution and decreases in value as more accurate reconstruction becomes possible.

\begin{wrapfigure}{r}{0.43\textwidth}
    \centering
    \vspace{-10pt}
    \includegraphics[width=0.43\textwidth]{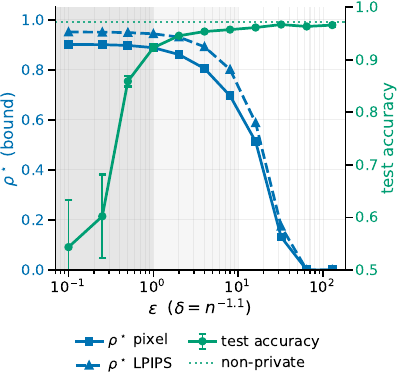}
    \caption{\textbf{CelebA: $95\%$-confidence reconstruction bounds and accuracy.} Pixel and LPIPS bounds $\rho^\star$ (left axis) and test accuracy (mean $\pm$ s.d.\ over three seeds; right axis), against $\varepsilon$ at $\delta=n_{\mathrm{train}}^{-1.1}$. The ratio $\rho^\star$ measures reconstruction error relative to the optimal prior-only guess: larger is more protective, and $\rho^\star=1$ means no improvement over that guess. Shading marks the privacy tiers of \citet{ponomareva2023dp}; other datasets in \Cref{fig:spectrum_all}.}
    \label{fig:spectrum_headline}
\end{wrapfigure}

We consider an adversary who observes the noisy parameter trajectory during training, knows every other training record, knows the population distribution from which the training records are drawn, and may use any, possibly randomised, reconstruction rule. 

Our framework separates the problem into a mechanism-side bound on the information revealed by training and a source-side bound on the uncertainty already present in the population training-data distribution. Rate--distortion theory combines them into a lower bound on the expected error of every reconstruction rule, biased or not; Figure~\ref{fig:framework} summarises this modular construction. For DP-SGD, the mechanism-side bound is computed before training from the sampling rates and noise multipliers, without using a reported $(\varepsilon,\delta)$ guarantee or analysing a trained model; it depends on neither the architecture nor the parameter count and can be added across training runs. The source-side entropy is lower-bounded from a finite sample with explicit confidence. The representation $\Psi$ determines which reconstruction differences matter: the identity map gives pixel-space error, LPIPS features give perceptual error, and other representations can target coarse, partial, or task-specific recovery. An orthogonal block construction extends the calculation to high-dimensional representations while preserving squared reconstruction error exactly.

\textbf{Our contributions include:}\newline
\noindent\textbf{A modular information-theoretic framework.}
We combine mechanism leakage, source uncertainty, and an information-to-distortion converse to bound every attack under a strong informed-adversary model (\Cref{thm:master}); each component can be tightened or replaced independently.\newline
\noindent\textbf{Knowledge of training-data distribution.} We assume that the adversary has complete knowledge of the population training-data distribution and can sample from it. In our bound it enters through a differential entropy that a finite sample-set can be used to bound from below with explicit confidence (\Cref{prop:coverage}).\newline
\noindent\textbf{Finite-sample, representation-flexible DP-SGD bounds.}
The mechanism-side input is computable before training; the source-side input is controlled from finite reference data. Choosing $\Psi$ adapts the guarantee to different reconstruction harms; with LPIPS, this gives, to our knowledge, the first reconstruction lower bound measured directly by LPIPS.\newline
\noindent\textbf{A priori and retrospective analysis.} The framework can select a DP-SGD noise multiplier to meet a prescribed reconstruction-error bound, or analyse the reconstruction-privacy of a published training process retrospectively from its reported schedule and a source sample.\newline
\noindent\textbf{Experiments and slack analysis.}
Across MNIST, CIFAR-10, CelebA and ImageNette, the bounds remain active in `reasonable' and `weak' ($\varepsilon>10$) privacy levels. We compare attack performance with our lower bound and diagnose the gap between guaranteed and achieved reconstruction error.

\section{Related Work}\label{sec:related}

\noindent\textbf{Reconstruction robustness.}
\citet{balle2022reconstructing} bound the probability that an informed adversary reconstructs a target within distortion $\eta$; subsequent work sharpens the accounting and extends it to settings including multiple targets \citep{hayes2023bounding,kaissis2023bounding,kulynych2026unifying,swanberg2025beyond,guerra2026understanding}. Existing instantiations specify the required prior-only success probability analytically or estimate it without finite-sample coverage. We instead estimate a source-entropy term with explicit confidence. Threshold guarantees suit secrets for which near-exact recovery is the harm, whereas our expected distortion captures harmful partial reconstruction of images. Further comparison appears in \Cref{tab:lineage}, \Cref{app:related_extended}.
\newline
\noindent\textbf{Expected-error bounds.}
Estimation-theoretic work lower-bounds reconstruction error using Fisher information and related inequalities \citep{Hannun2021MeasuringDL,guo2022bounding,Chaudhuri2024GuaranteesOC,maeng2023bounding}. The deployable Cram\'er--Rao and Hammersley--Chapman--Robbins bounds in this line apply to unbiased reconstructors and therefore exclude attacks exploiting the source prior. Our bound covers arbitrary reconstructors and is normalised by the optimal prior-only error; see \Cref{app:unbiased} for a detailed comparison.
\newline
\noindent\textbf{Information-theoretic antecedents.}
Conditional mutual information has been used to quantify information about individual samples revealed by learning algorithms \citep{raginsky2016information,feldman2018calibrating,cuff2016differential}, including through sequential decompositions for noisy iterative algorithms \citep{pensia2018generalization}. Distortion-based secrecy originates in rate--distortion theory \citep{Schieler2013RatedistortionTF}. \citet{tan2024defending} consider the ambient entropy-power inequality (related to \Cref{thm:master}) using it as motivation for minimising one of the terms it contains. This is geared toward a defence against (rather than lower bound for) reconstruction in federated learning, and neither estimates the entropy nor addresses its looseness in high dimension.

\section{Method}\label{sec:method}

Our framework has three components. We first specify the reconstruction adversary and the representation in which error is measured. We then convert an upper bound on per-record information leakage into a lower bound on reconstruction distortion. Finally, we derive a blockwise form that remains useful when the chosen representation is high-dimensional.

\subsection{Setting and threat model}\label{sec:setting}

We first formalise what the adversary observes, what side information it possesses, and how reconstruction quality is measured. Let $X=(x_1,\ldots,x_n)$ be a training dataset of $n$ samples in $\R^d$, write $X^{-i}$ for $X$ with the $i$-th row removed, and let $\mathcal A$ be a possibly stochastic training algorithm producing the full released parameter trajectory
\[
    \theta:=\theta_{0:T}=(\theta_0,\ldots,\theta_T)
\]
from $X$. Our running instance is DP-SGD \citep{abadi2016deep}, but the bound sees the mechanism only through the scalar $\gamma_i$ of \Cref{def:cmi}; DP-SGD enters solely in \Cref{sec:bound_gamma}.

\begin{assumption}[Finite-precision data]\label{asm:discrete}
Each training sample $x_i$ takes values in a finite subset $\mathcal Q\subset\R^d$, obtained by representing each of the $d$ components at $b_x$ bits of precision. Concretely, each component takes values in a set of at most $2^{b_x}$ levels with spacing $\Delta=2^{-b_x}\cdot c$ for some scale parameter $c>0$.
\end{assumption}

\begin{assumption}[I.I.D.\ source]\label{asm:iid}
The rows of $X$ are drawn i.i.d.\ from a per-sample distribution $P_x$ on $\mathcal Q$, so the dataset has law $P_X=P_x^{\otimes n}$.
\end{assumption}

These are all the reconstruction bound requires. The accountant (\Cref{sec:bound_gamma}) and entropy estimator (\Cref{sec:entropy}) state any additional assumptions where they arise. \Cref{asm:iid} is violated by preprocessing computed across the training set and then applied back to it; see \Cref{app:data_handling}.

\paragraph{Threat model.}
We adopt the adversary against which DP-SGD's own guarantee is proved \citep{abadi2016deep}. The adversary knows the architecture, $\mathcal A$ and its hyperparameters, every non-target training record $X^{-i}$, the training set size $n$, and the population distribution $P_x$ from which training records are drawn. It observes the full released parameter trajectory $\theta_{0:T}$, including every intermediate noisy update, and may exploit all of this information through a measurable, possibly randomised reconstruction rule
\[
    \hat x_i=g(\theta_{0:T},X^{-i}).
\]
A bound against it also covers weaker adversaries, such as one observing only $\theta_T$ or lacking $X^{-i}$. Expectations are over $X\sim P_x^{\otimes n}$ and all randomness of training, reconstruction and dither, so the guarantee is averaged over the source, as in existing reconstruction guarantees formulated relative to a prior.

\begin{definition}[Per-sample CMI]\label{def:cmi}
The information (which we measure in bits) about the target revealed by the full released parameter trajectory, conditional on all non-target records, is
\[
    \gamma_i:=I(x_i;\theta_{0:T}\mid X^{-i}).
\]
\end{definition}

The quantity $\gamma_i$ measures what the trajectory carries about $x_i$ once the rest of the dataset is known. Under \Cref{asm:iid}, the joint distribution of $(x_i,\theta_{0:T},X^{-i})$ is the same for every $i$. Because DP-SGD treats every index alike under Poisson subsampling, $\gamma_i$ does not depend on $i$; we write $\gamma$ for the common value and retain the index only in per-sample statements.

\begin{lemma}[Per-sample leakage chain]\label{lem:cmi_chain}
Under \Cref{asm:iid}, fix $m\geq1$ and let $\hat u_i=g(\theta_{0:T},X^{-i})\in\R^m$ for a measurable, possibly randomised $g$ whose internal randomness is independent of $x_i$ given $(\theta_{0:T},X^{-i})$. Then
\[
    I(x_i;\hat u_i)\leq I(x_i;\theta_{0:T},X^{-i})=I(x_i;\theta_{0:T}\mid X^{-i})=:\gamma_i.
\]
\end{lemma}

\begin{proof}[Proof sketch]
The inequality is the data-processing inequality along $x_i\to(\theta_{0:T},X^{-i})\to\hat u_i$. The equality follows from the chain rule $I(x_i;\theta_{0:T},X^{-i})=I(x_i;X^{-i})+I(x_i;\theta_{0:T}\mid X^{-i})$, whose first term vanishes under \Cref{asm:iid}. Full proof in \Cref{app:cmi_chain}.
\end{proof}

Hence any reconstruction carries at most $\gamma_i$ bits of mutual information about the target.

\subsection{From information leakage to distortion}\label{sec:leak_to_dist}

We now convert the information budget $\gamma_i$ into unavoidable reconstruction error. The private object is $x_i$, the representation $\Psi(x_i)$ specifies the aspect of that record whose reconstruction is to be controlled, and a dithered version of $\Psi(x_i)$ is an analytical device that permits a differential-entropy bound for finite-precision data.

The rate--distortion bound requires a continuous source entropy, whereas \Cref{asm:discrete} makes $x_i$ discrete. We therefore add independent dither after applying $\Psi$. This makes differential entropy well defined, and the dither's contribution to squared error can later be subtracted exactly. The construction is internal to the analysis: the model trains on the raw $X$, and neither $\Psi$ nor the dither is applied to the training data or observed by the adversary.

\begin{definition}[Dithered representation]\label{def:dither}
Call $\Psi:\R^d\to\R^m$ \emph{admissible} if, when random, its randomness is independent of $X$, of $\mathcal A$, and of the attack, with statements then holding conditionally on its draw; deterministic maps qualify. Fix a \emph{dither scale} $s>0$. The \emph{dithered representation} of $x_i$ under $\Psi$ is
\begin{equation}\label{eq:dither_def}
    \zT_i:=\Psi(x_i),
    \qquad
    \tilde z^{\,\Psi}_i:=\zT_i+W_i^{(m)},
    \qquad
    W_i^{(m)}\iid\mathcal N(0,s^2I_m),
\end{equation}
where the $W_i^{(m)}$ are drawn independently across $i$ and of everything above. Thus $\zT_i$, $\tilde z^{\,\Psi}_i$, and $\zhatT_i$ denote respectively the target representation, its dithered form, and the adversary's output, all in $\R^m$; the ambient case $\Psi=\mathrm{id}$ is written $\tilde x_i$ and $\hat x_i$.
\end{definition}

The bound holds for every $s>0$. The dither contribution is removed exactly at a cost that grows with $s$, while larger $s$ reduces the bias of the entropy estimator supporting the bound; see \Cref{app:why}. For every $\Psi$, the dithered representation has a density and finite entropy (\Cref{app:dither_regularity}). Write $\Sigma_\Psi:=\Cov(\Psi(x_i))$, and abbreviate $\Sigma:=\Cov(x_i)$.

\paragraph{Distortion--rate function.}
For an $\R^m$-valued source $Y$ under squared error, the \emph{distortion--rate function}, $D_Y(R)$, satisfies the following \emph{Shannon lower bound} (SLB) \citep{koch2016shannon}:
\begin{equation}\label{eq:slb_multi}
        D_Y(R):=\inf\{\E\|Y-\hat Y\|^2:I(Y;\hat Y)\leq R\}\geq\frac{m}{2\pi e}\,2^{2(h(Y)-R)/m}.
\end{equation}
Equality holds at every $R$ when $Y$ is isotropic Gaussian \citep{cover2006elements}.

\paragraph{The reconstruction ratio.}
Absolute error is difficult to compare across datasets and representations. We therefore normalise by the optimal reconstruction that ignores the released parameter trajectory and uses only the source distribution. For a constant guess $a\in\R^m$,
\[
    \E\|\Psi(x_i)-a\|^2=\tr\Sigma_\Psi+\|\E \Psi(x_i)-a\|^2,
\]
so the optimal prior-only reconstruction is the prior mean $a=\E \Psi(x_i)$, with error $\tr\Sigma_\Psi$. We report each guarantee through the \emph{reconstruction ratio}
\begin{equation}\label{eq:rho_def}
    \rho_\Psi(g):=\frac{\sqrt{\E\|\zT_i-\zhatT_i\|^2}}{\rmseprior(\Psi)},
    \qquad
    \rmseprior(\Psi):=\sqrt{\tr\Sigma_\Psi},
    \qquad
    \nu_\Psi^2:=\frac{ms^2}{\tr\Sigma_\Psi},
\end{equation}
dropping the subscript $\Psi$ in the ambient case. Since the prior mean is always available, $\inf_g\rho_\Psi(g)\leq1$, and a valid bound never exceeds one (\Cref{rem:rho_le_one}).

\paragraph{Choosing the representation.}
The map $\Psi$ specifies which aspects of the record are treated as privacy-sensitive. Taking $\Psi=\mathrm{id}$ measures pixel-space reconstruction; a coordinate selection measures recovery of a region, channel, or attribute; and downsampling or frequency truncation measures coarse-scale recovery. Taking $\Psi=\lpipsmap$, the concatenated unit-normalised weighted feature maps of LPIPS \citep{zhang2018unreasonable}, makes $\|\zT_i-\zhatT_i\|^2$ the perceptual distortion itself. More generally, any measurable map fixed independently of the protected data and training run, including a frozen learned encoder, is admissible. The leakage bound does not change with $\Psi$; only the source entropy and prior-only reconstruction error in the chosen representation must be recomputed. 

We can now combine any valid leakage bound with the source entropy in the chosen representation.

\begin{theorem}[Reconstruction lower bound under a representation]
\label{thm:master}
Let Assumptions~\ref{asm:discrete}--\ref{asm:iid} hold, let $\Psi:\R^d\to\R^m$ be admissible, and let $s>0$. Every possibly randomised $\zhatT_i=g(\theta_{0:T},X^{-i})$ taking values in $\R^m$ satisfies
\begin{equation}\label{eq:master_abs}
    \E\bigl[\|\zT_i-\zhatT_i\|^2\bigr]
    \geq
    \left(
       \frac{m}{2\pi e}\,2^{2(h(\tilde z^{\,\Psi}_i)-\gamma_i)/m}
       -ms^2
    \right)_+,
\end{equation}
equivalently, on the ratio scale of \eqref{eq:rho_def},
\begin{equation}\label{eq:master_rho}
    \rho_\Psi(g)^2
    \geq
    \left(
       \frac{m}{2\pi e\,\tr\Sigma_\Psi}
       2^{2(h(\tilde z^{\,\Psi}_i)-\gamma_i)/m}
       -\nu_\Psi^2
    \right)_+,
    \qquad
    \nu_\Psi^2:=\frac{ms^2}{\tr\Sigma_\Psi}.
\end{equation}
\end{theorem}

The lower bound increases with uncertainty in the represented source and decreases with the information revealed by training.

\begin{proof}[Proof sketch]
By \Cref{lem:cmi_chain} and the data-processing inequality along $\tilde z^{\,\Psi}_i\to x_i\to\zhatT_i$, $I(\tilde z^{\,\Psi}_i;\zhatT_i)\leq\gamma_i$. The dithered representation, unlike the discrete source, has a density and finite entropy, so the SLB controls $\E\|\tilde z^{\,\Psi}_i-\zhatT_i\|^2$ at rate $\gamma_i$. Because the dither is centred and independent of $\zT_i-\zhatT_i$, its error $ms^2$ can then be subtracted exactly. Full proof in \Cref{sec:thm-single-proof}.
\end{proof}

\subsection{Orthogonal block tightening}\label{sec:rp-tight}

\Cref{thm:master} is valid in any dimension, but can be weak and difficult to estimate for high-dimensional, non-Gaussian image sources. We therefore rotate the chosen representation and partition it into low-dimensional orthogonal blocks. Squared reconstruction error is preserved exactly, while each entropy calculation is performed in only $k$ dimensions.

Let $U\in\R^{m\times m}$ be orthogonal, let $N_k:=m/k$ for $k\mid m$, let $P_b\in\R^{k\times m}$ select block $b\in\{1,\ldots,N_k\}$, and set $\Pi_b:=\sqrt{N_k}P_bU$. Each block carries independent dither of scale $s$ as in \Cref{def:dither}, so the entropies below depend on $s$ as well as $k$; we suppress $s$ in the notation. Define the block-averaged entropy
\begin{equation}\label{eq:hbar_U}
    \bar h_U(k)
    :=\frac{1}{N_k}\sum_{b=1}^{N_k}h\bigl(\tilde z_i^{\,\Pi_b\Psi}\bigr),
    \qquad
    \tilde z_i^{\,\Pi_b\Psi}
    :=\Pi_b\Psi(x_i)+W_{i,b}^{(k)}.
\end{equation}

Block size introduces a trade-off: smaller blocks tend to carry more entropy per coordinate but each pays the full leakage $\gamma_i$, a factor $2^{-2\gamma_i/k}$, while larger blocks dilute this charge at the cost of lower entropy. The following corollary optimises $k$ over a pre-specified grid.

\begin{corollary}[Orthogonal-block instantiation]\label{cor:rp}
Let Assumptions~\ref{asm:discrete}--\ref{asm:iid} hold, let $\Psi:\R^d\to\R^m$ be admissible with $\tr\Sigma_\Psi>0$, let $k$ divide $m$, and let $U\in\R^{m\times m}$ be orthogonal, with $\Pi_b$ and $\bar h_U(k)$ as in \eqref{eq:hbar_U}. For every $s>0$ and every possibly randomised $\zhatT_i=g(\theta_{0:T},X^{-i})\in\R^m$, writing now $\nu_\Psi^2:=ks^2/\tr\Sigma_\Psi$,
\begin{equation}\label{eq:rp_min}
    \rho_\Psi(g)^2
    \geq
    \left(
       \frac{k}{2\pi e\,\tr\Sigma_\Psi}
       2^{2(\bar h_U(k)-\gamma_i)/k}
       -\nu_\Psi^2
    \right)_+,
\end{equation}
and hence, over any grid $\mathcal G$ of pairs $(k,\nu_\Psi^2)$ with $k\mid m$,
\begin{equation}\label{eq:rp_max}
    \rho_\Psi(g)
    \geq
    \ratiostar_\Psi
    :=\max_{(k,\nu_\Psi^2)\in\mathcal G}
    \left(
       \frac{k}{2\pi e\,\tr\Sigma_\Psi}
       2^{2(\bar h_U(k)-\gamma_i)/k}
       -\nu_\Psi^2
    \right)_+^{1/2}.
\end{equation}
The ambient case $\Psi=\mathrm{id}$ and $m=d$ gives the pixel bound $\ratiostar$; taking $\Psi=\lpipsmap$ gives the perceptual bound $\rhostarlp$.
\end{corollary}

\begin{proof}[Proof sketch]
Each $\Pi_b\circ\Psi$ is admissible, so \Cref{thm:master} applies blockwise to $\Pi_b\zhatT_i$. The maps $\Pi_b$ preserve squared norm on average, so averaging the $N_k$ block bounds returns the error in $\R^m$. Jensen's inequality for convex $t\mapsto2^{2t/k}$ replaces the average of block exponentials by the exponential of $\bar h_U(k)$, while the dither, removed exactly per block, contributes $ks^2$ once. Each pair $(k,\nu_\Psi^2)$ is separately valid, hence so is the maximum over the pre-specified grid. Proof in \Cref{app:representation}.
\end{proof}

We take $U=m^{-1/2}H_mV^\top$, where $V$ is an eigenbasis of $\Sigma_\Psi$ and $H_m$ is a Hadamard matrix, so that every block carries the same variance $\tr\Sigma_\Psi$ and every coordinate mixes all eigendirections equally; see \Cref{rem:choice_of_U}. Halving $k$ typically raises the entropy per coordinate but doubles the leakage charge $\gamma/k$; doubling $k$ has the opposite effect. The grid therefore selects the scale at which the entropy gain outweighs the per-coordinate leakage penalty, and the dither cost $\nu^2$ is optimised jointly with $k$.

\section{The a priori accountant}\label{sec:bound_gamma}

We next supply the mechanism-side input to \Cref{thm:master}. For Poisson-subsampled DP-SGD, the schedule $(q_t,\sigma_t)_{t=1}^T$ upper-bounds the per-sample leakage $\gamma=I(x_i;\theta_{0:T}\mid X^{-i})$. The calculation uses neither the trained parameters nor the realised gradients, and does not take a reported $(\varepsilon,\delta)$ guarantee as input. It is independent of the parameter dimension $p$ in value and cost.

At step $t$ of DP-SGD \citep{abadi2016deep}, a Poisson-subsampled minibatch $B_t$ includes each index independently with probability $q_t$. Per-sample gradients are clipped to $L_2$-norm $C$,
\[
    \bar g_t(x_j)
    :=\nabla_\theta\ell(\theta_{t-1};x_j)
      \min\left\{1,
      \frac{C}{\|\nabla_\theta\ell(\theta_{t-1};x_j)\|}\right\},
\]
and, with learning rate $\eta_t$, noise multiplier $\sigma_t$, and expected batch size $b_t=q_tn$, the released parameter trajectory evolves as
\begin{equation}\label{eq:dpsgd_update}
    \theta_t
    =\theta_{t-1}
     -\frac{\eta_t}{b_t}
      \left(
        \sum_{j\in B_t}\bar g_t(x_j)+\sigma_tC\,\xi_t
      \right),
    \qquad
    \xi_t\sim\mathcal N(0,I_p).
\end{equation}
We write $\mathcal H_{t-1}:=\theta_{0:t-1}$. We take $C=1$ throughout: rescaling the update by $C$ leaves mutual information invariant, so the bound, which depends only on $(q_t,\sigma_t)$, holds for any clipping norm.

\begin{theorem}[Dominating-pair leakage bound]\label{thm:gc_canonical}
For DP-SGD with updates \eqref{eq:dpsgd_update}, hyperparameters $(q_t,\sigma_t)_{t=1}^T$, and data-independent initialisation,
\begin{equation}\label{eq:Kprv_def}
    \gamma
    \leq \frac{1}{\ln 2}\sum_{t=1}^T K_t^{\mathrm{PRV}},
    \qquad
    K_t^{\mathrm{PRV}}
    :=\KL\bigl(P_t^{(1)}\,\|\,Q_t^{(1)}\bigr),
\end{equation}
where
\[
    P_t^{(1)}
    :=(1-q_t)\,\mathcal N(0,\sigma_t^2)
      +q_t\,\mathcal N(1,\sigma_t^2),
    \qquad
    Q_t^{(1)}:=\mathcal N(0,\sigma_t^2).
\]
The pair $(P_t^{(1)},Q_t^{(1)})$ is the standard one-dimensional dominating pair for the Poisson-subsampled Gaussian mechanism with sampling rate $q_t$, noise multiplier $\sigma_t$, and unit sensitivity \citep{mironov2019renyi,zhu2022optimal}. Moreover, $K_t^{\mathrm{PRV}}$ is the expectation under $P_t^{(1)}$ of the privacy-loss random variable $\ln(\mathrm dP_t^{(1)}/\mathrm dQ_t^{(1)})$ used by the PRV accountant \citep{gopi2021numerical}.
\end{theorem}

\begin{proof}[Proof sketch]
The chain rule decomposes trajectory leakage into per-step conditional mutual informations. Conditional on the past and non-target records, each update is a Bernoulli-gated Gaussian channel whose divergence is bounded by the standard one-dimensional dominating pair. Summing these bounds and converting nats to bits gives \eqref{eq:Kprv_def}; see \Cref{app:accountant}.
\end{proof}

Thus the reconstruction analysis consumes DP-SGD through a single scalar information budget, obtained by summing schedule-dependent per-step charges. The factor $1/\ln2$ converts the KL divergences in \eqref{eq:Kprv_def} from nats to the bits used in \Cref{thm:master}. Each $K_t^{\mathrm{PRV}}$ is a one-dimensional integral, evaluated to an error negligible relative to the finite-sample margins of \Cref{sec:entropy} (\Cref{app:exp_spectrum}). The accountant can therefore be computed before training, or retrospectively from reported sampling rates and noise multipliers.

\section{A confidence bound on the representation entropy}\label{sec:entropy}

\Cref{sec:bound_gamma} supplies the mechanism-side leakage bound. We now obtain the source-side entropy bound required by \Cref{cor:rp} from finite reference data. Because \eqref{eq:rp_min} is increasing in the entropy it uses, a lower confidence bound on $\bar h_U(k)$ gives a valid, conservative reconstruction bound.

By \Cref{def:dither}, each $\tilde z_i^{\,\Pi_b\Psi}$ is a finitely supported source convolved with Gaussian dither of a chosen scale $s$. We use the plug-in estimator of \citet{goldfeld2020convergence}, namely the entropy of the empirical mixture. For a sample $x_{1:n}$, define
\begin{equation}\label{eq:plugin_def}
    \hat f_n(y;\Pi_b)
    :=\frac{1}{n}\sum_{j=1}^n
      \phi_s\bigl(y-\Pi_b\Psi(x_j)\bigr),
    \qquad
    \hat h_n(\Pi_b)
    :=h\bigl(\hat f_n(\cdot;\Pi_b)\bigr),
    \qquad
    \hat h_n
    :=\frac{1}{N_k}\sum_{b=1}^{N_k}\hat h_n(\Pi_b),
\end{equation}
where $\phi_s$ is the $k$-dimensional dither density. We evaluate this entropy using the Monte Carlo integrator of \citet{goldfeld2020convergence} with $n_{\mathrm{MC}}$ dither replicates per centre and denote the result by $\hat h_{n,\mathrm{MC}}$. The estimator is minimax rate-optimal for this estimand, and its bias is downward \citep[Eq.~(70)]{goldfeld2020convergence}, which is the conservative direction required here. The finite-sample margins below are new. Throughout, $n$ denotes the size of the sample used for entropy estimation, which need not be the training set (see \Cref{app:data_handling}).

The confidence statement is simultaneous over the candidate block sizes and dither scales, so the final pair may be selected after estimation without invalidating coverage.

\begin{proposition}[Confidence bound on the representation entropy] \label{prop:coverage} Let Assumptions~\ref{asm:discrete}--\ref{asm:iid} hold. Fix $\alpha\in(0,1)$, an admissible $\Psi:\R^d\to\R^m$ with $\tr\Sigma_\Psi>0$, and the grid $\mathcal G$ of \eqref{eq:rp_max}. Let $U$ and the reference trace $\tr\hat\Sigma_\Psi$ be fixed using held-out data, independently of the estimation sample $x_{1:n}$, as is $\Psi$, and for each $(k,\nu^2)\in\mathcal G$ set \( s=\nu (\tr\hat\Sigma_\Psi/k)^\frac{1}{2} \). For the closed-form margins $t_{\mathrm{MC}},t_{\mathrm{MD}}>0$ defined in \Cref{app:margins}, and $\hat h_{n,\mathrm{MC}}$, $t_{\mathrm{MC}}$ and $\bar h_U(k)$ evaluated at the dither scale $s$ of each pair, with probability at least $1-\alpha$ over the i.i.d.\ estimation sample $x_{1:n}\sim P_x$ and Monte Carlo noise, given $D_\perp$,
\begin{equation}\label{eq:h_certified} 
    \bar h_{\mathrm{LB}}(k) :=\hat h_{n,\mathrm{MC}}-t_{\mathrm{MC}}-t_{\mathrm{MD}} \leq \bar h_U(k) \qquad \text{for all }(k,\nu^2)\in\mathcal G. 
\end{equation} 
On the same event, every reconstruction admitted by the threat model in \Cref{sec:setting} satisfies \begin{equation}\label{eq:abs_certified} 
    \E\bigl[\|\zT_i-\zhatT_i\|^2\bigr] \geq \max_{(k,\nu^2)\in\mathcal G} \left( \frac{k}{2\pi e} 2^{2(\bar h_{\mathrm{LB}}(k)-\gamma)/k} -ks^2 \right)_+ , \end{equation} 
and, if in addition $\tr\hat\Sigma_\Psi\geq\tr\Sigma_\Psi$, \begin{equation}\label{eq:rho_certified} 
    \rho_\Psi(g) \geq \max_{(k,\nu^2)\in\mathcal G} \left( \frac{k}{2\pi e\,\tr\hat\Sigma_\Psi} 2^{2(\bar h_{\mathrm{LB}}(k)-\gamma)/k} -\nu^2 \right)_+^{1/2}. \end{equation} 
\end{proposition}

\begin{proof}[Proof sketch]
Cantelli's inequality controls the Monte Carlo error, McDiarmid's inequality controls the sampling error, and the downward plug-in bias requires no margin. A union bound over $\mathcal G$ permits post-estimation selection of $(k,\nu^2)$; see \Cref{app:entropy}.
\end{proof}

Both entropy margins are finite-sample (\Cref{app:margins}). For normalisation we replace $\tr\Sigma$ by its held-out estimate $\tr\hat\Sigma$; dividing instead by its finite-sample upper confidence bound lowers $\rho^\star$ by at most $0.03$, except on ImageNette, whose $D_\perp$ is small (\Cref{tab:trsigma}). The underlying absolute-distortion bound does not involve $\tr\Sigma$ and keeps its finite-sample coverage.

\begin{figure}[t]
    \centering
    \includegraphics[width=\textwidth]{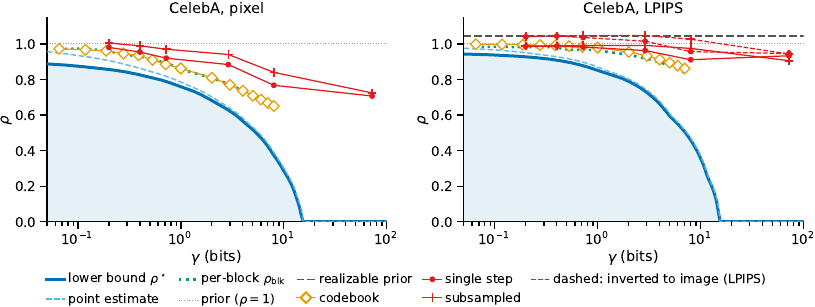}
    \caption{\textbf{CelebA: bounds, attacks and slack decomposition.} Pixel (left) and LPIPS (right) reconstruction ratios at leakage $\gamma$. The bound $\rho^\star$, block-specific bound $\rho_{\mathrm{blk}}$, achieved codebook-release error $\rho_{\mathrm{cb}}$, and attack errors against a single full-batch noisy-gradient release and a $100$-step subsampled DP-SGD trajectory are shown. The large $\rho^\star$-to-$\rho_{\mathrm{blk}}$ gap is caused by charging every block the full leakage; the close $\rho_{\mathrm{blk}}$-to-$\rho_{\mathrm{cb}}$ match shows that little slack remains in the rate--distortion step. Details are in \Cref{app:exp_compressor}.}
    \label{fig:tightness}
    \vspace{-6pt}
\end{figure}

\section{Experiments}\label{sec:experiments}

We ask two questions: what reconstruction protection does a realistic DP-SGD run guarantee at a given utility cost, and what explains the gap between this guarantee and a strong reconstruction attack? \Cref{fig:framework} gives the conceptual framework; the pipeline below shows how we compute its components in both experiments.

\paragraph{Computing the bound.}
\begin{center}
\setlength{\tabcolsep}{2pt}
\small
\begin{tabular}{c@{\hspace{3pt}}c@{\hspace{3pt}}c@{\hspace{3pt}}c@{\hspace{3pt}}c@{\hspace{3pt}}c@{\hspace{3pt}}c}
\fbox{\parbox[c]{0.18\linewidth}{\centering
\textbf{Account}\\[-1pt]
$(q_t,\sigma_t)_{t=1}^T\mapsto\gamma$\\[-2pt]
{\scriptsize\Cref{thm:gc_canonical}}}}
& $\longrightarrow$ &
\fbox{\parbox[c]{0.18\linewidth}{\centering
\textbf{Calibrate}\\[-1pt]
$D_\perp\mapsto U,\nu,\tr\hat\Sigma_\Psi$\\[-2pt]
{\scriptsize\Cref{app:data_handling}}}}
& $\longrightarrow$ &
\fbox{\parbox[c]{0.24\linewidth}{\centering
\textbf{Estimate}\\[-1pt]
$D_{\mathrm{est}}\mapsto\bar h_{\mathrm{LB}}(k)$ \\[-1pt]
for all $(k,\nu^2)\in\mathcal G$\\[-2pt]
{\scriptsize\Cref{prop:coverage}}}}
& $\longrightarrow$ &
\fbox{\parbox[c]{0.16\linewidth}{\centering
\textbf{Optimise}\\[-1pt]
maximise over $\mathcal G$\\[-2pt]
{\scriptsize\eqref{eq:rho_certified}}}}
\end{tabular}
\end{center}

\paragraph{Datasets and representations.} We evaluate \eqref{eq:rho_certified}, the finite-sample instantiation of \Cref{cor:rp} obtained using \Cref{prop:coverage}. The main text reports representative pixel-space and LPIPS results, while the appendix extends the evaluation across MNIST, CIFAR-10, CelebA and ImageNette (\Cref{app:exp_data,app:exp_phi}). Both representations use the same schedule leakage $\gamma$; only their source entropy and prior-only error differ. Deployment results are reported against $(\varepsilon,\delta)$, with $\delta=n_{\mathrm{train}}^{-1.1}$. Further experimental and data-handling details appear in \Cref{app:data_handling,app:exp}.

\paragraph{Protection and utility.} For each target $\varepsilon$, we calibrate $\sigma$, compute the resulting leakage $\gamma$ and bound $\rho_\Psi^\star$, and train a model under the same DP-SGD schedule. We denote the pixel and LPIPS bounds by $\rhostarpx$ and $\rhostarlp$. Thus \Cref{fig:spectrum_headline,fig:spectrum_all} pair the reconstruction bound computable before training with the test accuracy subsequently achieved by each schedule. On CelebA, $\rhostarpx$ remains above $0.5$ up to $\varepsilon=16$ ($\rhostarpx=0.513$), where accuracy is $96.1\pm0.3\%$, one point below the non-private baseline; $\rhostarlp$ is higher there ($\rhostarlp=0.588$). At $\varepsilon=8$, $\rhostarpx$ lies between $0.68$ and $0.75$ across all four datasets. The bounds remain non-trivial for $\varepsilon>10$, but weaken as leakage approaches the source uncertainty certifiable from finite data. We report against $\varepsilon$ for familiarity; equal $(\varepsilon,\delta)$ does not imply equal $\gamma$ or equal reconstruction protection.

\paragraph{Attack comparison.} We use a single reconstruction attack throughout: the prior-aware attack of \citet{hayes2023bounding}, applied to the linear model of \citet{schwethelm2024visual}, whose per-example gradient is the input (\Cref{app:exp_attack}). The attack uses all non-target records and released updates and estimates the source prior from a shadow sample excluding the target. At matched leakage $\gamma$, we apply the attack to a one-step full-batch release, $(q,T)=(1,1)$, and a $100$-step subsampled trajectory, $(q,T)=(0.02,100)$. These are the one-step and multi-step series in \Cref{fig:tightness,fig:tightness_mnist}. Both mean errors remain above the same bound; the subsampled release is harder to exploit because target participation is hidden. Individual targets may fall below the bound because it bounds population-average expected error (\Cref{fig:sandwich_full}). Numerical results appear in \Cref{tab:attack_results}.

\paragraph{Where does the slack arise?} \Cref{fig:tightness} gives a clear answer: the principal identifiable looseness comes from charging every block the full leakage, while little comes from the rate--distortion step. To show this, at leakage $R=\gamma$ we compare DP-SGD with a reconstruction-optimised codebook release that indexes the nearest of $2^R$ representative points; each bit released by this codebook is much more helpful for reconstruction than bits leaked from DP-SGD. We need two codebooks, one for rates above one bit and one for sub-bit rates (\Cref{app:exp_compressor}). We evaluate the chain: \( \rho^\star \leq \rho_{\mathrm{blk}} \leq \rho_{\mathrm{cb}} \leq \rho_{\mathrm{dpsgd}} \). Here $\rho^\star$ is the bound, $\rho_{\mathrm{blk}}$ charges each block only its actual leakage, and $\rho_{\mathrm{cb}}$ and $\rho_{\mathrm{dpsgd}}$ are attack errors against the codebook and DP-SGD releases. The first gap is $0.10$--$0.33$ in $\rho$, whereas the second is at most $0.03$, showing that the rate--distortion step is close to tight after correcting the block charge. The remaining gap is mechanism-specific, though it may also include accountant and attack looseness. Separately, we also introduce some slack by taking a confidence bound; what we lose by not using the entropy point estimate is shown in \Cref{fig:tightness}. See \Cref{app:exp_compressor} for details.

\section{Conclusion}\label{sec:conclusion}

We introduced a modular framework that lower-bounds population-average reconstruction error under an informed adversary. For DP-SGD it combines a schedule-level leakage bound with a finite-sample source-entropy bound in an analyst-chosen representation. Across the image datasets studied, the pixel and LPIPS bounds remain non-trivial across commonly used privacy regimes, including settings with $\varepsilon>10$.\newline
\textbf{Limitations.} The guarantee is in expectation under the source and does not preclude more accurate reconstruction of particular records or subpopulations. The bound has looseness, nearly all looseness can be attributed to the full leakage penalty being charged to every block and not accounting for the fact that some released bits are more helpful for reconstruction than others. Observing the full noisy trajectory directly covers trajectory-exposing settings but is conservative when only the final model is released. The theory also excludes cross-record preprocessing and auxiliary information beyond the stated threat model.\newline
\textbf{Future work.} The modular construction allows mechanism-specific leakage bounds, source-specific converses, other distortion measures, and representations targeting particular harms. Conditioning leakage and source uncertainty on auxiliary information could model better-informed adversaries, while joint source and leakage terms may yield group-reconstruction guarantees. Finally, using unreleased iterates or model-specific gradient structure could tighten a deployment-specific guarantee at the cost of architecture independence.

\newpage
\subsection*{AI use statement}

In this work, we only used generative AI tools (LLM-based chat and coding assistants) for the following tasks: assisting with the mathematical development, providing feedback on research methodology, as a writing aid and writing and running code. The mathematical assistance consisted of checking proofs for errors, suggesting intermediate steps in proofs, making proofs more formal, keeping notation consistent and general brainstorming. All mathematical statements were thoroughly checked multiple times and line-by-line by the authors. As a writing aid it was used to help shorten text and increase clarity. We used AI to help write much of the code. The authors checked the code by eye and employed sanity checks, plots and tests during development to make sure that the code is valid. We take responsibility for the final content of this work, including text, claims or artifacts produced with the aid of generative AI.

\subsection*{Ethics statement}

This work is intended as a tool to help practitioners choose and assess privacy protections, by bounding how accurately a training record can be reconstructed from a model trained with DP-SGD. To test the bound we run reconstruction attacks, adapted from the prior-aware attack of \citet{hayes2023bounding}, against a deliberately simple model on public benchmark datasets.

\subsection*{Reproducibility statement}

\Cref{alg:certify} gives the full procedure for computing the bound, and \Cref{app:data_handling} specifies the data splits it uses. \Cref{app:exp} gives all remaining experimental details: dataset preprocessing, the parameter grid and estimator settings, the architectures, hyperparameters and noise calibration for the DP-SGD runs, and the construction of the attack and the codebook mechanism. All datasets are public. Code for all experiments is available at \url{https://github.com/h-0-0/info_reconstruction_bound}.

\subsubsection*{Acknowledgments}
We would like to thank His Majesty’s Government for contributing toward Robert F. Allison’s and Henry Bourne's funding during the course of this work. We also thank the EPSRC Centre for Doctoral Training in Computational Statistics and Data Science (COMPASS) for funding Henry Bourne’s PhD studentship. This work was carried out using the computational facilities of the Advanced Computing Research Centre, University of Bristol - http://www.bristol.ac.uk/acrc/.

\bibliography{iclr2027_conference}
\bibliographystyle{iclr2027_conference}

\appendix

\section{Extended related work}\label{app:related_extended}

Here we provide further details on related reconstruction bound work, this literature is helpfully summarized by \Cref{tab:lineage} which organises each distinct lineage (ReRo and expected distortion) in order of publication date. 

\begin{table}[h]
\centering\scriptsize
\setlength{\tabcolsep}{2.5pt}
\renewcommand{\arraystretch}{1.15}
\newcolumntype{Y}{>{\raggedright\arraybackslash}X}
\caption{The reconstruction-guarantee lineage in order of first public release, split by bounded metric: probability guarantees ($\Pr$) above the double rule, expected-error guarantees ($\E$) below. \textit{Threat Model} gives background knowledge (\textbf{Informed}: adversary knows all training records but the target; \textbf{Uninformed}: no background dataset knowledge; \textbf{Target-informed}: adds target-specific auxiliary data) and access (\textbf{Final release}: the mechanism's final output only; \textbf{All iterates}: intermediate updates or gradients). \textit{Prior Integration} indicates whether a prior over the target enters the guarantee itself, and \textit{Prior Assumption and Source} whether it is assumed (\textbf{Analytic}) or approximated from data (\textbf{Empirical}); structural assumptions on the estimator or data space are noted in the same cell. \textit{Baseline success} unites disparate notations for the maximum probability of guessing the target from prior knowledge alone; \citet{guerra2026understanding} instead use the collision probability $\kappa_\pi = \Pr_{Z,Z'\sim\pi}[Z=Z']$, which coincides with a baseline success only under a uniform prior. $\Delta\Pr$ denotes a guarantee stated as an advantage --- the increase in success attributable to the target's participation --- rather than a raw probability. Experimental cells list representative settings, not an exhaustive inventory. Abbreviations: LDP/RDP/$f$-DP (Local, Rényi, and $f$-differential privacy), CMI (conditional mutual information), LCB (lower confidence bound), HCRB (Hammersley--Chapman--Robbins bound), CRB (Cramér--Rao bound), FIL (Fisher information leakage), dFIL (diagonal FIL), RR (randomized response).}
\label{tab:lineage}
\begin{tabularx}{\textwidth}{@{}>{\raggedright\arraybackslash}p{19mm} Y Y Y Y Y Y@{}}
\toprule
& Threat Model & Mechanism Leakage Bound & Bounded Object (Metric) & Prior Integration into Bound & Prior Assumption and Source & Experimental Instantiation \\
\midrule
\citet{bhowmick2018protection} & Uninformed, all iterates & $\varepsilon$-LDP: $e^{\varepsilon}$ scaling & $\Pr[\text{error} \le \eta]$ & Baseline success $p$ & Analytic: Uniform/Zipfian & MNIST, CIFAR-10, Flickr, Reddit: Accuracy gap \\
\midrule
\citet{balle2022reconstructing} & Informed, final release & RDP/$\varepsilon$-DP: $e^{\varepsilon}$ scaling & $\Pr[\text{error} \le \eta]$ & Baseline success $\kappa$ & Analytic: Uniform/Gaussian & NA \\
\midrule
\citet{xiao2023pac} & Uninformed, final release & PAC Privacy: $I(X; \mathcal{M}(X)+B)$ & $\Pr[\text{error} \le \eta]$ & Baseline success $\delta$ & Analytic: Uniform population & CIFAR-10, MNIST: Noise magnitude \\
\midrule
\citet{guo2023analyzing} & Informed, final release & RDP + Fano: Generalized Fano & $\Pr[\text{exact}]$ & Baseline success $p^*$, Entropy $H(p)$ & Analytic: Prob.\ vector $p$ & Synthetic (RR, Gaussian), IWPC, GPT-2: Advantage, MSE \\
\midrule
\citet{hayes2023bounding} & Informed, all iterates & ReRo: Blow-up function & $\Pr[\text{exact}]$ & Baseline success $\kappa$ & Analytic: Uniform shortlist & MNIST, CIFAR-10: Success probability \\
\midrule
\citet{kaissis2023bounding} & Informed, final release & $f$-DP: Closed-form / Edgeworth & $\Pr[\text{exact}]$ & Baseline success $\kappa$ & Analytic: Uniform $1/n$ prior & Numerical: success bound vs.\ $\kappa$ at published DP-SGD hyperparameters \\
\midrule
\citet{swanberg2025beyond} & Informed ($n$ targets), final release & Pure/$f$-DP: $e^{\varepsilon}$ tail & $\Pr[\text{count of hits}]$ & Baseline success $\pi$ & Analytic: Product distribution $\mathcal{D}_1 \otimes \dots \otimes \mathcal{D}_n$ & LLM canaries, tabular: Advantage, Success \\
\midrule
\citet{kulynych2026unifying} & Informed, final release & $f$-DP: Trade-off (Thm.~3.1); TV (Thm.~3.3) & $\Pr[\text{attack success}]$ & Baseline success $base$ & Analytic: Scalar worst-case & US Census, GPT-2: Advantage, Accuracy \\
\midrule
\citet{guerra2026understanding} & Target-informed, final release & TV (Thm.~4.2); density difference (Thm.~4.3) & $\Delta\Pr[\text{error} \le \eta]$ & Collision probability $\kappa_\pi$ & Uniform/Data-estimated frequencies & MNIST, Adult, Porto: RAD \\
\midrule\midrule
\citet{guo2022bounding} & Informed, final release & RDP: HCRB (Thm.~1); FIL: CRB (Thm.~2) & $\E[\text{MSE}]$ & No prior; domain geometry / Fisher info & Analytic: bounded $\mathcal{Z}$; regularity; non-vacuous only for unbiased attacks & MNIST, CIFAR-10: MSE lower bound; dFIL vs.\ reconstruction \\
\midrule
This paper & Informed, all iterates & CMI: KL of dominating pair & $\E[\text{MSE}]$ & Differential entropy & Empirical: LCB from sample & MNIST--ImageNette: MSE, LPIPS \\
\bottomrule
\end{tabularx}
\end{table}

\paragraph{The information quantities in the ReRo literature.} Two threshold entries consume an information quantity, and neither is $\gamma_i$. \citet{xiao2023pac} consume the dataset-level $I(X;\mathcal M(X))$, computed after the fact by simulating the mechanism, and in their individual form a conditional mutual information given one other user's record rather than the leave-one-out conditioning on all of $X^{-i}$ that the informed adversary requires. \citet{guo2023analyzing} consume the R\'enyi mutual information $I_\alpha(X;Y)$ bounded by the RDP parameter, together with the Shannon entropy of a supplied probability vector, spent through a generalised Fano inequality on the probability of exact recovery rather than on distortion.

\paragraph{Expected-error bounds.}\label{app:unbiased}
Both bounds of \citet{guo2022bounding} carry an explicit bias-variance trade-off, but their deployable forms, $\mathrm{MSE}\ge\sum_i\mathrm{diam}_i(\mathcal Z)^2/(4d(e^\varepsilon-1))$ from RDP, and $\mathrm{MSE}\ge d/\tr I_h(z)$ from Fisher information, are the unbiased case, and the consequences of this assumption are visible in their own experiments. On $[0,1]^{784}$ an adversary that ignores the model and predicts the midpoint of each coordinate attains $\mathrm{MSE}\le 1/4$, well inside the $\mathrm{MSE}=1$ line they draw as ``perfect privacy''. Their per-sample Fisher bounds all exceed $1$, so the accounting assures every training sample as safe. Conversely, the informed-adversary attack of \citet{balle2022reconstructing} falls below the Fisher bound on some samples, which they attribute to the attack incurring ``a high bias when the sample is hard to reconstruct''.

The van Trees bound of \citet{maeng2023bounding} is the closest in form to ours as its prior distribution is estimated from a sample, and it needs a smooth density, so they perturb the images with Gaussian noise before score matching. The differences are in the direction of the estimate and where the smoothing lives. Theirs is a point estimate, which they note ``can be incorrect'' and should be read ``as a guideline'', and which is breached by their own attack when that perturbation is small, whereas ours is a one-sided confidence bound whose estimator is biased in the safe direction. Their perturbation is applied to the inputs themselves, where our dither is internal to the analysis and removed exactly (\Cref{thm:master}).

\section{Proof of the reconstruction lower bound}\label{sec:thm-single-proof}

\subsection{Proof of \Cref{lem:cmi_chain}}\label{app:cmi_chain}
\begin{proof}
The first inequality is the Data Processing Inequality (DPI) on $x_i \to (\theta, X^{-i}) \to \hat{u}_i$, valid since $\hat u_i$ depends on $x_i$ only through $(\theta, X^{-i})$ and randomness independent of $x_i$ given it. For the equality, recall the chain rule for mutual information \cite[Theorem~2.5.2]{cover2006elements}: for any random variables $A, B, C$ we have $I(A;\, B, C) = I(A;\, C) + I(A;\, B \mid C)$. Applying this with $A = x_i$, $B = \theta$, and $C = X^{-i}$ gives $I(x_i;\, \theta, X^{-i}) = I(x_i;\, X^{-i}) + I(x_i;\, \theta \mid X^{-i})$, and $I(x_i;\, X^{-i}) = 0$ under \Cref{asm:iid} leaves $I(x_i;\, \theta, X^{-i}) = I(x_i;\, \theta \mid X^{-i}) =: \gamma_i$.
\end{proof}

\subsection{Regularity of the dithered representation}\label{app:dither_regularity}

\begin{lemma}[Regularity of the dithered representation]\label{lem:dither_regularity}
Let \Cref{asm:discrete} hold, let $\Psi:\R^d\to\R^m$ be admissible (conditioning on its draw where $\Psi$ is random), let $s>0$, and let $\tilde z^{\,\Psi}_i = \Psi(x_i) + W^{(m)}_i$ be as in \Cref{def:dither}. Write $\phi_s$ for the $\mathcal N(0,s^2 I_m)$ density. Then
\begin{enumerate}
    \item[(i)] $\tilde z^{\,\Psi}_i$ has density $f_\Psi(y) = \sum_{v\in\mathcal Q}\Pr(x_i=v)\,\phi_s\bigl(y-\Psi(v)\bigr)$, a finite Gaussian mixture;
    \item[(ii)] $\Cov(\tilde z^{\,\Psi}_i) = \Sigma_\Psi + s^2 I_m$, and $\E\|\tilde z^{\,\Psi}_i\|^2 < \infty$;
    \item[(iii)] $h(\tilde z^{\,\Psi}_i)$ is finite and $h(\tilde z^{\,\Psi}_i) = h(W^{(m)}) + I\bigl(\Psi(x_i);\,\tilde z^{\,\Psi}_i\bigr) \ge h(W^{(m)}) = \tfrac m2\log_2(2\pi e s^2)$.
\end{enumerate}
\end{lemma}

\begin{proof}
\emph{(i).} With probability $\Pr(x_i=v)$ the representation is $\Psi(v)+W^{(m)}_i$, whose density is $\phi_s(\cdot-\Psi(v))$, so $\tilde z^{\,\Psi}_i$ has the mixture density $f_\Psi$.

\emph{(ii).} Covariances of independent summands add, and $\Cov(W^{(m)}_i)=s^2I_m$. Expanding $\E\|\Psi(x_i)+W^{(m)}_i\|^2$, the cross term vanishes since $W^{(m)}_i$ is centred and independent of $\Psi(x_i)$, and $\E\|W^{(m)}_i\|^2 = ms^2$; $\E\|\Psi(x_i)\|^2$ is finite as $\Psi(x_i)$ takes finitely many values.

\emph{(iii).} Each $\phi_s(\cdot-\Psi(v))$ is at most $(2\pi s^2)^{-m/2}$ and the weights in (i) sum to one, so $f_\Psi \le (2\pi s^2)^{-m/2}$ everywhere; hence $-\log_2 f_\Psi \ge \tfrac m2\log_2(2\pi s^2)$ wherever $f_\Psi>0$ and $h(\tilde z^{\,\Psi}_i) \ge \tfrac m2\log_2(2\pi s^2) > -\infty$. Gaussian maximum entropy at the covariance of (ii) \citep[Thm.~8.6.5]{cover2006elements} gives $h(\tilde z^{\,\Psi}_i) < \infty$. Since $\Psi(x_i)$ is discrete and $W^{(m)}_i$ is Gaussian and independent of it, $h(\tilde z^{\,\Psi}_i) = h(W^{(m)}) + I(\Psi(x_i);\tilde z^{\,\Psi}_i)$ \citep[eq.~(53)]{goldfeld2020convergence}, and the inequality follows from non-negativity of mutual information.
\end{proof}

\subsection{Proof of \Cref{thm:master}}\label{app:thm_master_proof}

\begin{proof}[Proof of \cref{thm:master}]
If $\E\bigl[\|\zT_i - \zhatT_i\|^2\bigr] = \infty$ the bound holds trivially, so assume throughout that it is finite. Since $\zT_i = \Psi(x_i)$ is bounded ($\mathcal Q$ is finite by \Cref{asm:discrete}), this gives $\E\|\zhatT_i\|^2 < \infty$, and hence $\E\|\zT_i - \zhatT_i\| < \infty$ by Cauchy--Schwarz.

Fix an index $i \in \{1, \dots, n\}$ and, where $\Psi$ is random, work conditionally on its draw.

\medskip\noindent\textbf{Step 1: Rate bound.}
For any (possibly randomised) reconstruction $\zhatT_i = g(\theta, X^{-i})$, applying \Cref{lem:cmi_chain} at $\hat u_i=\zhatT_i$ gives
\begin{equation}\label{eq:step1_rate}
    I(x_i;\, \hat z^{\,\Psi}_i) \;\leq\; \gamma_i.
\end{equation}

\medskip\noindent\textbf{Step 2: Markov chain through the dithered representation.}
Recall $\tilde{z}^{\,\Psi}_i = \Psi(x_i) + W^{(m)}_i$ (\Cref{def:dither}). Conditionally on $x_i$, the variable $\tilde{z}^{\,\Psi}_i$ is a function of $W^{(m)}_i$ alone (working, where $\Psi$ is random, conditionally on its draw, which admissibility permits since that draw is independent of the data, of $\mathcal A$, and of the attack), while $\zhatT_i = g(\theta, X^{-i})$ is a function of $(\theta, X^{-i})$ and the attack's internal randomness. By \Cref{def:dither} the dither $W^{(m)}_i$ is drawn independently of all of these jointly, so $\zhatT_i \perp\!\!\!\perp \tilde{z}^{\,\Psi}_i \mid x_i$, i.e.\ $\tilde{z}^{\,\Psi}_i \to x_i \to \hat z^{\,\Psi}_i$ is a Markov chain, and the data-processing inequality \citep[Theorem~2.8.1]{cover2006elements} with \eqref{eq:step1_rate} gives
\begin{equation}\label{eq:step2_dpi}
    I(\tilde{z}^{\,\Psi}_i;\, \hat z^{\,\Psi}_i) \;\leq\; I(x_i;\, \hat z^{\,\Psi}_i)
    \;\leq\; \gamma_i.
\end{equation}
Now we have bounded the mutual information between the two terms whose expected MSE we would like to bound, we can use the distortion-rate bound from \eqref{eq:slb_multi},
\begin{equation}\label{eq:step2_rd}
    \E\!\left[\|\tilde{z}^{\,\Psi}_i - \hat z^{\,\Psi}_i\|^2\right]
    \;\geq\; D_{\tilde{z}^{\,\Psi}_i}(\gamma_i).
\end{equation}

\medskip\noindent\textbf{Step 3: Shannon lower bound.}
By \Cref{lem:dither_regularity}, $\tilde{z}^{\,\Psi}_i$ is absolutely continuous on $\R^m$ with finite differential entropy and finite second moment, so the SLB applies at rate $R = \gamma_i$:
\begin{equation}\label{eq:step3_slb}
    D_{\tilde{z}^{\,\Psi}_i}(\gamma_i) \;\geq\;
    \frac{m}{2\pi e}\, 2^{\,2(h(\tilde{z}^{\,\Psi}_i) - \gamma_i)/m}.
\end{equation}

\medskip\noindent\textbf{Step 4: Exact removal of the dither.}
Since $\tilde{z}^{\,\Psi}_i - \zhatT_i = (\zT_i - \zhatT_i) + W^{(m)}_i$, expanding the square gives
\[
    \E\bigl[\|\tilde{z}^{\,\Psi}_i - \hat z^{\,\Psi}_i\|^2\bigr]
    \;=\; \E\bigl[\|\zT_i - \zhatT_i\|^2\bigr]
    \;+\; 2\,\E\bigl[\langle \zT_i - \zhatT_i,\, W^{(m)}_i\rangle\bigr]
    \;+\; \E\bigl[\|W^{(m)}_i\|^2\bigr].
\]
By \Cref{def:dither} the dither $W^{(m)}_i$ is independent of $\zT_i - \zhatT_i$ and has zero mean. So, expanding coordinatewise gives $\E\langle \zT_i - \zhatT_i,\, W^{(m)}_i\rangle = \sum_{j=1}^m \E[(\zT_i - \zhatT_i)_j]\,\E[W^{(m)}_{i,j}] = 0$, the factorisation valid since $\E\|\zT_i - \zhatT_i\| < \infty$ by the assumption at the start of the proof. The last term is $m s^2$ exactly. Hence
\begin{equation}\label{eq:dither_identity}
    \E\bigl[\|\zT_i - \zhatT_i\|^2\bigr]
    \;=\; \E\bigl[\|\tilde{z}^{\,\Psi}_i - \hat z^{\,\Psi}_i\|^2\bigr] \;-\; m s^2 ,
\end{equation}
Substituting \eqref{eq:step2_rd} and \eqref{eq:step3_slb} into \eqref{eq:dither_identity}, and noting the left-hand side is non-negative, proves \eqref{eq:master_abs}.

\medskip\noindent\textbf{Ratio form.}
Assume $\tr\Sigma_\Psi > 0$, i.e.\ $\Psi(x_i)$ is not a.s.\ constant, so that $\rho_\Psi$ of \eqref{eq:rho_def} is defined. Dividing the inequality of Step 4 --- before the positive part is taken --- through by $\rmseprior(\Psi)^2 = \tr\Sigma_\Psi > 0$ and using $\nu_\Psi^2 = m s^2/\tr\Sigma_\Psi$ from \eqref{eq:rho_def} gives
\[
    \rho_\Psi(g)^2 \;\geq\; \frac{m}{2\pi e\,\tr\Sigma_\Psi}\, 2^{\,2(h(\tilde z^{\,\Psi}_i) - \gamma_i)/m} \;-\; \nu_\Psi^2 ,
\]
the left-hand side being $\rho_\Psi(g)^2$ by \eqref{eq:rho_def} and the reduction at the start of the proof. Since $\rho_\Psi(g)^2 \ge 0$, the two combine to give the positive part in \eqref{eq:master_rho}.

\end{proof}

\section{Proof of \Cref{cor:rp}}\label{app:representation}

\begin{proof}[Proof of \Cref{cor:rp}]
Fix an orthogonal matrix $U$. Fix a pair $(k,\nu_\Psi^2) \in\mathcal G$, where $k$ divides $m$, and let $s = \nu_\Psi\sqrt{\tr\Sigma_\Psi/k}$ be the corresponding dither scale.
 
\medskip\noindent\textbf{Step 1: Per-block bound.}
Given any reconstruction $\zhatT_i = g(\theta, X^{-i})\in\R^m$, the induced block estimate $\hat z^{\,\Pi_b\Psi}_i := \Pi_b\zhatT_i$ relies only on $(\theta, X^{-i})$ and the attack's internal randomness. The composite $\Pi_b\circ\Psi:\R^d\to\R^k$ is admissible in the sense of \Cref{def:dither}: $\Pi_b$ is deterministic, so the only randomness in $\Pi_b\circ\Psi$ is that of $\Psi$, which admissibility of $\Psi$ already places independent of $X$, of $\mathcal A$ and of the attack. \Cref{thm:master} therefore applies at the representation $\Pi_b\circ\Psi$ with $m$ replaced by $k$, at the same leakage $\gamma_i$:
\begin{equation}\label{eq:block_bound_fixed}
    \E\bigl[\|\Pi_b(\Psi(x_i) - \zhatT_i)\|^2\bigr] \;\geq\; \frac{k}{2\pi e}\, 2^{\,2(h(\tilde z^{\,\Pi_b\Psi}_i) - \gamma_i)/k} \;-\; k s^2 , \qquad \tilde z^{\,\Pi_b\Psi}_i = \Pi_b\Psi(x_i) + W^{(k)}_{i,b}.
\end{equation}
These are $N_k$ separate applications of the theorem, each block picking fresh i.i.d.\ dither, so no joint condition across the $W^{(k)}_{i,b}$ is required. They are combined below by summation alone.
 
\medskip\noindent\textbf{Step 2: Transfer to $\R^m$.}
The $P_b$ group the $m$ rows of $U$ into $N_k$ disjoint blocks of $k$, so by Pythagoras $\sum_b\|P_b u\|^2 = \|u\|^2$ for every $u \in \R^m$. Taking $u = Uv$ and using $\|Uv\| = \|v\|$, as $U$ is orthogonal,
\begin{equation}\label{eq:block_isometry}
    \frac{1}{N_k}\sum_{b=1}^{N_k}\|\Pi_b v\|^2 \;=\; \sum_{b=1}^{N_k}\|P_b Uv\|^2 \;=\; \|v\|^2 \qquad\text{for every } v \in \R^m ,
\end{equation}
the factor $N_k$ in $\Pi_b = \sqrt{N_k}\,P_b U$ being what turns the partition sum into an average. Applying \eqref{eq:block_isometry} at $v = \Psi(x_i) - \zhatT_i$ and taking expectations,
\begin{equation}\label{eq:ambient_identity}
    \E\bigl[\|\Psi(x_i) - \zhatT_i\|^2\bigr] \;=\; \frac{1}{N_k}\sum_{b=1}^{N_k}\E\bigl[\|\Pi_b(\Psi(x_i) - \zhatT_i)\|^2\bigr].
\end{equation}
 
\medskip\noindent\textbf{Step 3: Averaging the block bounds.}
Substituting \eqref{eq:block_bound_fixed} into \eqref{eq:ambient_identity} and pulling out the factors common to all blocks,
\[
    \E\bigl[\|\Psi(x_i) - \zhatT_i\|^2\bigr] \;\geq\; \frac{k}{2\pi e}\, 2^{-2\gamma_i/k}\, \frac{1}{N_k}\sum_{b=1}^{N_k} 2^{\,2h(\tilde z^{\,\Pi_b\Psi}_i)/k} \;-\; k s^2 \;\geq\; \frac{k}{2\pi e}\, 2^{\,2(\bar h_U(k) - \gamma_i)/k} \;-\; k s^2 ,
\]
the last step by Jensen on the convex map $t \mapsto 2^{2t/k}$ applied to the uniform average over blocks, $\bar h_U(k)$ by \eqref{eq:hbar_U}.
 
\medskip\noindent\textbf{Step 4: Ratio form.}
Dividing by $\rmseprior(\Psi)^2 = \tr\Sigma_\Psi$ and using $\nu_\Psi^2 = k s^2/\tr\Sigma_\Psi$ gives
\[
    \rho_\Psi(g)^2 \;\geq\; \frac{k}{2\pi e\,\tr\Sigma_\Psi}\, 2^{\,2(\bar h_U(k) - \gamma_i)/k} \;-\; \nu_\Psi^2 ,
\]
and since $\rho_\Psi(g)^2 \ge 0$ as well, the two combine to give the positive part in \eqref{eq:rp_min}. The pair $(k,\nu_\Psi^2)$ was arbitrary and is chosen by us, not the adversary: $g$ is fixed first, so every pair in $\mathcal G$ bounds the same $\rho_\Psi(g)^2$ and the largest of those bounds holds too, proving \eqref{eq:rp_max}.
\end{proof}

\begin{remark}[Choosing the rotation]\label{rem:choice_of_U}
Every orthogonal $U$ is valid; the choice affects only $\bar h_U(k)$, hence tightness. We write $\Psi=\mathrm{id}$; for general $\Psi$ replace $d,\Sigma$ by $m,\Sigma_\Psi$. By \eqref{eq:block_isometry} the block variances $v_b := \tr\Sigma_{\Pi_b}$ average to $\tr\Sigma$ for every $U$, and Gaussian maximum entropy \citep[Thm.~8.6.5]{cover2006elements} caps each block entropy by $f(v_b) := \tfrac k2\log_2\bigl(2\pi e\,(v_b+ks^2)/k\bigr)$. Since $f$ is concave, Jensen's inequality gives $N_k^{-1}\sum_b f(v_b) \le f\bigl(N_k^{-1}\sum_b v_b\bigr) = f(\tr\Sigma)$, with equality exactly when every $v_b = \tr\Sigma$. We take $U := d^{-1/2}H_dV^\top$. By the spectral theorem for real symmetric matrices, $\Sigma = V\Lambda V^\top$ with $V$ orthogonal and $\Lambda = \mathrm{diag}(\lambda_1,\dots,\lambda_d)$, so $V^\top x_i$ has uncorrelated coordinates with variances $\lambda_1,\dots,\lambda_d$. The Hadamard matrix $H_d$ satisfies $H_dH_d^\top = dI_d$ with every entry $\pm1$, so $d^{-1/2}H_d$ is orthogonal and mixes the uncorrelated coordinates with equal weight $1/d$, giving every output coordinate variance $\tr\Sigma/d$ and hence every block $v_b = \tr\Sigma$, for every $k\mid d$ at once. The equal-weight mixing also places each coordinate in the regime where projections of a high-dimensional source are near-Gaussian \citep{diaconis1984}, which the SLB rewards; we measure rather than assume this, through the saturation value of \Cref{rem:rho_le_one}. $V$ is computed on data held out from the estimation sample, as \Cref{prop:coverage} requires. In our experiments $V$ is the eigenbasis for the MNIST, CIFAR-10 and CelebA pixel bounds. For the ImageNette pixel bound ($|D_\perp|=1{,}000<d$) and for every LPIPS bound we take instead $V^\top=D$, a fixed random diagonal sign flip, which needs no data; the block variances then equal $\tr\Sigma$ only on average over the signs.

Hadamard matrices exist only at orders $1$, $2$ and multiples of $4$, and are known at every such order we use, by Sylvester's doubling, Paley's construction and Kronecker products of these (\Cref{app:exp_entropy}). If no Hadamard matrix of order $d$ is available we zero-pad to the nearest order at which one is: the padded coordinates are deterministically zero, so a bound for the padded representation transfers to the original, at a cost in tightness set by the padding fraction.
\end{remark}

\begin{remark}[The bound never exceeds the prior baseline]\label{rem:rho_le_one}
\Cref{cor:rp} holds for every reconstruction $g$, and the prior-mean guess is one with $\rho_\Psi(g) = 1$ by \eqref{eq:rho_def}, so $\ratiostar_\Psi \le 1$ for every $U$, $k$ and $s$: the bound never claims more protection than releasing nothing affords, and the same holds for the reported bound when it is normalised by an upper bound on $\tr\Sigma$, since $\bar h_{\mathrm{LB}}(k) \le \bar h_U(k)$ only lowers it. Equality at $\gamma = 0$ requires $2^{2\bar h_U(k)/k} = 2\pi e\,(\tr\Sigma + ks^2)/k$, which by Gaussian maximum entropy \citep[Thm.~8.6.5]{cover2006elements} holds only when every block is isotropic Gaussian with variance $\tr\Sigma$; the saturation value below $1$ in \Cref{fig:spectrum_headline,fig:spectrum_all} ($0.87$--$0.91$ in pixel space across the four datasets and $0.92$--$0.96$ in LPIPS space across the three, at $\gamma=0$) is what the source's departure from that ideal costs, together with the dither and the estimation margins.
\end{remark}

\section{Proofs for the a priori accountant}\label{app:accountant}

The proof of \Cref{thm:gc_canonical} has three ingredients, presented in the order they are used: the chain rule splits $\gamma$ over steps (\Cref{app:acc_chain}); each step is reduced to a Bernoulli-gated Gaussian channel (\Cref{app:acc_channel}); and the channel's leakage is bounded by the one-dimensional reduction of \citet{mironov2019renyi}. \Cref{app:acc_proof} assembles them into the final result.

\subsection{Trajectory chain rule}\label{app:acc_chain}

\begin{lemma}[Trajectory chain rule]\label{lem:traj_chain}
With data-independent initialisation, the per-sample CMI decomposes exactly as
\begin{equation}\label{eq:traj_chain}
    \gamma_i \;=\; \sum_{t=1}^T I(x_i;\,\theta_t\mid\mathcal H_{t-1}, X^{-i}).
\end{equation}
\end{lemma}

\begin{proof}
Applying the chain rule for mutual information to $\theta = \theta_{0:T}$ one iterate at a time,
\[
    \gamma_i \;=\; I(x_i;\,\theta_{0:T}\mid X^{-i}) \;=\; I(x_i;\,\theta_0\mid X^{-i}) \;+\; \sum_{t=1}^T I(x_i;\,\theta_t\mid \mathcal H_{t-1}, X^{-i}).
\]
The first term vanishes because $\theta_0$ is drawn independently of $X$, hence of $x_i$ given $X^{-i}$.
\end{proof}

\subsection{Per-step channel reduction}\label{app:acc_channel}

\begin{lemma}[Per-step channel reduction]\label{lem:per_step_channel}
Fix a step $t$ with $\eta_t\neq0$ and let $\Omega := (\mathcal H_{t-1},\, X^{-i},\, B_t\setminus\{i\})$ collect the past, the other records, and the other minibatch indices. Let $G := \bar g_t(x_i)$ be the target's clipped gradient and
\begin{equation}\label{eq:gated_channel}
    Y := \mathbf 1_t\, G + \sigma_t\xi_t, \qquad \mathbf 1_t := \mathbf 1[i\in B_t]\sim\mathrm{Bernoulli}(q_t),\ \xi_t \sim \mathcal N(0, I_p),
\end{equation}
the \emph{Bernoulli-gated Gaussian channel} with input $G$. Then
\begin{equation}\label{eq:per_step_bound}
    I(x_i;\,\theta_t\mid \mathcal H_{t-1}, X^{-i}) \;\le\; I(x_i;\,\theta_t\mid\Omega) \;=\; I(G;\,Y\mid\Omega).
\end{equation}
\end{lemma}

\begin{proof}
\emph{Equality.} Given $\Omega$, the iterate $\theta_{t-1}$ and the other members' gradient sum $S := \sum_{j\in B_t\setminus\{i\}}\bar g_t(x_j)$ are determined, so the update of \eqref{eq:dpsgd_update} reads
\[
    \theta_t \;=\; \theta_{t-1} - \tfrac{\eta_t}{b_t}\bigl(Y + S\bigr),
    \qquad\text{equivalently}\qquad
    Y \;=\; -\tfrac{b_t}{\eta_t}\bigl(\theta_t - \theta_{t-1}\bigr) - S ,
\]
so $\theta_t$ is a function of $(Y,\Omega)$ and $Y$ a function of $(\theta_t,\Omega)$. Expanding $I(x_i;\,Y,\theta_t\mid\Omega)$ by the chain rule \citep[Thm.~2.5.2]{cover2006elements} in both orders,
\[
    I(x_i;\,\theta_t\mid\Omega) + \underbrace{I(x_i;\,Y\mid\theta_t,\Omega)}_{=\,0}
    \;=\; I(x_i;\,Y\mid\Omega) + \underbrace{I(x_i;\,\theta_t\mid Y,\Omega)}_{=\,0},
\]
each conditional term vanishing because the variable it measures is determined by its conditioning. Hence $I(x_i;\theta_t\mid\Omega) = I(x_i;Y\mid\Omega)$. Next, $G = \bar g_t(x_i)$ is a function of $(x_i,\Omega)$, the gradient being evaluated at $\theta_{t-1}$, so
\[
    I(x_i;\,Y\mid\Omega) \;=\; I(x_i, G;\,Y\mid\Omega) \;=\; I(G;\,Y\mid\Omega) + I(x_i;\,Y\mid G,\Omega),
\]
and the last term vanishes: $Y = \mathbf 1_t G + \sigma_t\xi_t$ with $(\mathbf 1_t,\xi_t)$ drawn independently of $(x_i,\Omega)$, so given $(G,\Omega)$ it is independent of $x_i$.

\emph{Inequality.} Poisson subsampling draws the inclusion indicators independently of the data and the past, so $B_t\setminus\{i\} \perp x_i \mid (\mathcal H_{t-1}, X^{-i})$. Expanding $I(x_i;\,\theta_t, B_t\setminus\{i\} \mid \mathcal H_{t-1}, X^{-i})$ by the chain rule in both orders,
\begin{align*}
    &\underbrace{I(x_i;\,B_t\setminus\{i\}\mid \mathcal H_{t-1}, X^{-i})}_{=\,0} + I(x_i;\,\theta_t\mid\Omega) \\
    &= I(x_i;\,\theta_t\mid \mathcal H_{t-1}, X^{-i}) + \underbrace{I(x_i;\,B_t\setminus\{i\}\mid \theta_t, \mathcal H_{t-1}, X^{-i})}_{\ge\,0},
\end{align*}
which is the inequality in \eqref{eq:per_step_bound}.
\end{proof}

\subsection{Proof of \Cref{thm:gc_canonical}}\label{app:acc_proof}
\begin{proof}
By \Cref{lem:traj_chain,lem:per_step_channel},
\[
    \gamma \;=\; \sum_{t=1}^T I(x_i;\,\theta_t\mid\mathcal H_{t-1}, X^{-i}) \;\le\; \sum_{t=1}^T I(G_t;\,Y_t\mid\Omega_t),
\]
so it suffices to show $I(G_t;Y_t\mid\Omega_t) \le K_t^{\mathrm{PRV}}$ at each step, with $K_t^{\mathrm{PRV}}$ as defined in \eqref{eq:Kprv_def}. Fix $t$ and drop it from the subscripts of $G$, $Y$ and $\Omega$.

\emph{Per-input output distribution.} Recall $P_{Y\mid z} = (1-q_t)P_0 + q_tP_z$ from \eqref{eq:Kprv_def}. Conditionally on $G = z$ and $\Omega = \omega$, the only randomness remaining in $Y = \mathbf 1_t G + \sigma_t\xi_t$ is the gate and the noise, which are independent of $(x_i, \Omega)$ --- Poisson subsampling draws $\mathbf 1_t$ independently of the data, the past and the other indices, and $\xi_t$ is fresh --- so
\[
    P_{Y\mid G=z,\,\Omega=\omega} \;=\; P_{Y\mid z} \qquad\text{for every } z \text{ and } \omega:
\]
with probability $1-q_t$ the target is excluded and $Y$ is pure noise, $P_0$; with probability $q_t$ it is included and $Y$ is noise about the clipped gradient, $P_z$. The input $G$ is a function of $(x_i,\Omega)$ and so is correlated with $\Omega$, but the channel from $G$ to $Y$ is not: only the input distribution varies with $\omega$. In particular $P_0$ is the distribution of $Y$ when the target is not sampled.

\emph{Dropping the conditioning.} Since $P_{Y\mid G,\Omega} = P_{Y\mid G}$, we have $Y \perp \Omega \mid G$, so expanding $I(G,\Omega;\,Y)$ by the chain rule in both orders,
\[
    I(G;\,Y) + \underbrace{I(\Omega;\,Y\mid G)}_{=\,0} \;=\; I(G;\,Y\mid\Omega) + \underbrace{I(\Omega;\,Y)}_{\ge\,0},
\]
hence $I(G;Y\mid\Omega) \le I(G;Y)$: the right-hand side is the mutual information of the channel \eqref{eq:gated_channel} at input distribution $P_G$, the marginal distribution of the target's clipped gradient, and conditioning on $\Omega$ only removes what $Y$ already reveals about $\Omega$.

\emph{Golden formula.} For distributions $P_X$, $P_{Y\mid X}$, $Q_Y$, write $\KL(P_{Y\mid X}\,\|\,Q_Y \mid P_X) := \E_{x\sim P_X}\bigl[\KL(P_{Y\mid X=x}\,\|\,Q_Y)\bigr]$ for the conditional divergence. By the golden formula \citep[Thm.~4.1]{polyanskiy2024information}, for any $Q_Y$,
\begin{equation}\label{eq:golden}
    \KL(P_{Y\mid G}\,\|\,Q_Y \mid P_G) \;=\; I(G;\,Y) \;+\; \KL(P_Y\,\|\,Q_Y) \;\ge\; I(G;\,Y).
\end{equation}

\emph{Leave-one-out reference.} Take $Q_Y = P_0$ in \eqref{eq:golden}. Unpacking the conditional divergence by its definition and using $P_{Y\mid G=z}=P_{Y\mid z}$,
\begin{equation}\label{eq:loo_ref}
    I(G;\,Y) \;\le\; \KL(P_{Y\mid G}\,\|\,P_0 \mid P_G) \;=\; \E_{z\sim P_G}\bigl[\KL(P_{Y\mid z}\,\|\,P_0)\bigr] \;\le\; \sup_{\|z\|\le 1}\KL(P_{Y\mid z}\,\|\,P_0),
\end{equation}
the last step because every value $G$ can take has norm at most one (\eqref{eq:dpsgd_update}).

\emph{Worst case over clipped inputs.} Write $\mathcal M$ for the Poisson-subsampled Gaussian mechanism at rate $q_t$, noise $\sigma_t$ and unit sensitivity. \citet[Thm.~4]{mironov2019renyi} show that for every dataset $S$, every record $x$ with $\|\bar g_t(x)\|\le1$, and every order $\alpha\ge1$,
\[
    D_\alpha\bigl(\mathcal M(S\cup\{x\})\,\big\|\,\mathcal M(S)\bigr) \;\le\; D_\alpha\bigl(P^{(1)}_t\,\big\|\,Q^{(1)}_t\bigr),
\]
the one-dimensional pair of \eqref{eq:Kprv_def}.\footnote{Theorem~4 of \citet{mironov2019renyi} bounds the R\'enyi divergence in both directions, $D_\alpha(\mathcal M(S)\,\|\,\mathcal M(S\cup\{x\}))$ and $D_\alpha(\mathcal M(S\cup\{x\})\,\|\,\mathcal M(S))$; we need only the latter. The theorem is stated for R\'enyi order $\alpha$, and the KL divergence is the case $\alpha=1$, which their Definition~1 includes by taking the limit $\alpha\to1$; nothing in their proof requires $\alpha>1$.} Take $S=\emptyset$ and $x$ with clipped gradient $z$: $\mathcal M(\emptyset)$ is pure noise, $P_0$, and $\mathcal M(\{x\})$ includes the record with probability $q_t$, giving $P_{Y\mid z}$. The other records of the actual training set do not appear because \Cref{lem:per_step_channel} has already conditioned them away through $\Omega$. At $\alpha=1$ the right-hand side is $K_t^{\mathrm{PRV}}$, so the supremum in \eqref{eq:loo_ref} is at most $K_t^{\mathrm{PRV}}$, hence
\begin{equation}\label{eq:prv_step}
    I(G;\,Y\mid\Omega) \;\le\; I(G;\,Y) \;\le\; K_t^{\mathrm{PRV}},
\end{equation}
with both sides in nats; dividing by $\ln 2$ to convert to bits and summing over $t$ completes the proof.

\end{proof}

\section{Details on the representation entropy confidence bound}\label{app:entropy}

\Cref{app:why} sets out what the bound needs from an entropy estimator and why the plug-in is the only candidate that supplies it; \Cref{app:bdd_diff,app:margins} then verify each requirement in turn, and \Cref{app:coverage} chains them into \Cref{prop:coverage}.

\subsection{What the bound needs from the entropy estimator}\label{app:why}

\Cref{cor:rp} consumes $\bar h_U(k)$ through the exponent $2^{\,2(\bar h_U(k)-\gamma)/k}$, maximised over a pair $(k,\nu^2)$ chosen from $\mathcal G$ after the data is seen. We supply it through a lower confidence bound $\bar h_{\mathrm{LB}}(k)\le\bar h_U(k)$, and the bound pays for the shortfall $\bar h_U(k)-\bar h_{\mathrm{LB}}(k)$ exactly as it pays for leakage, since a shortfall of $t$ bits costs the bounded ratio a factor $2^{-2t/k}$. The shortfall has two parts, whatever the point estimate is biased by and whatever is subtracted from it to make the bound hold at the stated confidence. What the bound needs is that both be controlled in a single one-sided, finite-sample statement, namely that $\bar h_{\mathrm{LB}}(k)\le\bar h_U(k)$ for every $(k,\nu^2)\in\mathcal G$ with probability at least $1-\alpha$. Only the lower side is used, and the union bound over $\mathcal G$ in \Cref{prop:coverage} needs a per-pair tail probability at the $n$ we have rather than in the limit.

Most entropy estimators do not supply this. Nearest-neighbour estimators \citep{kozachenko1987sample,berrett2019efficient} have a bias depending on the unknown density, so no finite-sample statement is available in either direction, and variational estimators of the mutual information into which $h(\tilde z^{\,\Pi_b})$ decomposes (\Cref{lem:dither_regularity}(iii)) are lower bounds in expectation only. The finite-sample error bounds that do exist for the smoothed setting, the risk bound of \citet[Thm.~2]{goldfeld2020convergence} and the confidence sequence of \citet[Cor.~21]{manole2023martingale}, are two-sided, and a two-sided bound must control the estimator's bias, which costs a prefactor exponential in the dimension and makes the bound vacuous at our sample sizes. We need only a lower bound, and the plug-in's bias is downward, so it needs no control at all; what remains is the fluctuation, which \Cref{lem:t_md} bounds at a constant that does not grow with the dimension.

The plug-in of \citet{goldfeld2020convergence} meets the requirement. Its error about $\bar h_U(k)$ splits into a bias and a fluctuation. The bias is downward at every $n$ \citep[eq.~(70)]{goldfeld2020convergence}, so it needs no margin, and this is what places the plug-in outside the two-sided barrier above, since it is one-sided by construction. The fluctuation is a function of the $n$ independent samples, the rotation and dither scale having been fixed on disjoint data (\Cref{app:data_handling}), and changing one sample moves it by at most $H_b(1/n)$ regardless of $k$ (\Cref{cor:bdd_diff_avg}), so McDiarmid gives a finite-sample margin of $\tilde O(n^{-1/2})$ bits carrying no $k$ (\Cref{lem:t_md}). The Monte Carlo integrator we evaluate it with has known conditional mean and variance given the sample (\Cref{lem:mc_moments}), so its error is given a margin the same way (\Cref{lem:t_mc}). It is capped, however: $\hat h_n\le h(W^{(k)})+\log_2n$ by \Cref{lem:mixture_sandwich}.

\subsection{Bounded differences for the plug-in}\label{app:bdd_diff}

\citet{goldfeld2020convergence} do not provide a lower confidence bound, they bound $\E[\hat h_n]$ and have no reason to control its fluctuation about that mean. Here we derive the constant needed for McDiarmid's inequality. It rests on a standard result about mixtures which we will re-derive here first for the reader's benefit.

\begin{lemma}[Entropy of a finite mixture]\label{lem:mixture_sandwich}
Let $f_1,\dots,f_r$ be densities on $\R^k$ with finite differential entropy, and let $\lambda$ be a probability vector on $\{1,\dots,r\}$. Then
\begin{equation}\label{eq:mixture_sandwich}
    \sum_{i=1}^r \lambda_i\, h(f_i) \;\le\; h\Bigl(\sum_{i=1}^r\lambda_i f_i\Bigr) \;\le\; \sum_{i=1}^r \lambda_i\, h(f_i) + H(\lambda),
\end{equation}
with $H(\lambda) := -\sum_i\lambda_i\log_2\lambda_i$ the Shannon entropy of the weights.
\end{lemma}
\begin{proof}
Let $A\sim\lambda$ and, given $A=i$, let $B\sim f_i$, so that $B$ has the mixture density and $h(B\mid A) = \sum_i\lambda_i h(f_i)$. The left inequality is $h(B\mid A)\le h(B)$, conditioning reducing entropy \citep[Thm.~8.6.1]{cover2006elements}. For the right, $h(B) = h(B\mid A) + I(B;A)$ and $I(B;A)\le H(A) = H(\lambda)$. That is, knowing which component a draw came from saves at most $H(\lambda)$ bits.
\end{proof}

\begin{lemma}[Bounded differences under one relocation]\label{lem:bdd_diff}
Fix $s>0$ and $\Pi:\R^d\to\R^k$. Let $x_{1:n}$ and $x'_{1:n}$ be two samples that agree except at one index $j$, and let $\hat f_n$ and $\hat f'_n$ be the corresponding mixtures \eqref{eq:plugin_def}: each is the equal-weight average of the $n$ Gaussians $\phi_s(\cdot-\Pi x_i)$, respectively $\phi_s(\cdot-\Pi x'_i)$. Then
\[
    \bigl|h(\hat f_n) - h(\hat f'_n)\bigr| \;\le\; H_b(1/n),
\]
with $H_b(p) := -p\log_2 p - (1-p)\log_2(1-p)$ the binary entropy in bits.
\end{lemma}
\begin{proof}
Let $g := \tfrac{1}{n-1}\sum_{i\ne j}\phi_s(\cdot - \Pi x_i)$ be the equal-weight mixture of the $n-1$ Gaussians the two share. Then
\[
    \hat f_n = \tfrac{n-1}{n}\,g + \tfrac1n\,\phi_s(\cdot - \Pi x_j), \qquad \hat f'_n = \tfrac{n-1}{n}\,g + \tfrac1n\,\phi_s(\cdot - \Pi x'_j),
\]
two-component mixtures with the same first component and the same weights $\lambda=(\tfrac{n-1}{n},\tfrac1n)$. Since a Gaussian's entropy depends only on its covariance, $h(\phi_s(\cdot-\Pi x_j)) = h(\phi_s(\cdot-\Pi x'_j)) = h(W^{(k)})$, so the weighted component entropies coincide for the two mixtures, $\sum_i\lambda_i h(f_i) = \tfrac{n-1}{n}h(g) + \tfrac1n h(W^{(k)}) =: L$ in both cases. \Cref{lem:mixture_sandwich} then places both $h(\hat f_n)$ and $h(\hat f'_n)$ in $[L,\,L+H(\lambda)]$, and $H(\lambda)=H_b(1/n)$.
\end{proof}

\begin{corollary}[Bounded differences for the block average]\label{cor:bdd_diff_avg}
Fix $s>0$ and a finite family $\Pi_1,\dots,\Pi_{N_k}$ of maps $\R^d\to\R^k$, all independent of $x_{1:n}$. The map $x_{1:n}\mapsto\hat h_n$ of \eqref{eq:plugin_def} satisfies the bounded-differences condition with constant $H_b(1/n)$ in each coordinate: for every $j$ and every $x_j' \in \mathcal Q$,
\[
    \bigl|\hat h_n(x_1,\dots,x_j,\dots,x_n) - \hat h_n(x_1,\dots,x_j',\dots,x_n)\bigr| \;\le\; H_b(1/n).
\]
The constant does not depend on $k$, $N_k$, $s$, or the $\Pi_b$.
\end{corollary}
\begin{proof}
Changing $x_j$ relocates exactly one centre of every $\hat f_n(\cdot\,;\Pi_b)$, so by \Cref{lem:bdd_diff} each $\hat h_n(\Pi_b)$ moves by at most $H_b(1/n)$, and an average of quantities each moving by at most $c$ moves by at most $c$.
\end{proof}

\subsection{The two margins}\label{app:margins}

Fix $(k,\nu^2)\in\mathcal G$, which fixes $s=\nu\sqrt{\tr\Sigma/k}$. Establishing \eqref{eq:h_certified} involves three quantities: the target $\bar h_U(k)$, a fixed property of the source; the plug-in $\hat h_n$ of \eqref{eq:plugin_def}, the exact block-averaged entropy of the empirical mixtures, which is random through the sample $x_{1:n}$ and has no closed form; and $\hat h_{n,\mathrm{MC}}$, the Monte Carlo approximation to $\hat h_n$ that we actually compute, which is random through the sample and through the Monte Carlo noise. Writing
\begin{equation}\label{eq:error_decomp}
    \hat h_{n,\mathrm{MC}} - \bar h_U(k) \;=\; \bigl(\hat h_{n,\mathrm{MC}} - \hat h_n\bigr) + \bigl(\hat h_n - \E[\hat h_n]\bigr) + \bigl(\E[\hat h_n] - \bar h_U(k)\bigr),
\end{equation}
the third term is a deterministic quantity and is non-positive for every $n$: it is the downward bias of the plug-in \citep[eq.~(70)]{goldfeld2020convergence}, averaged over blocks. The first two are random and each is given a margin, $t_{\mathrm{MC}}$ and $t_{\mathrm{MD}}$, that it exceeds with probability at most $\alpha_1/|\mathcal G|$ and $\alpha_2/|\mathcal G|$ respectively, where $\alpha_1+\alpha_2=\alpha$. The two are separated by conditioning on the sample: with $x_{1:n}$ fixed, $\hat h_n$ is a constant and the first term is pure integration error, whose conditional mean and variance are known (\Cref{lem:mc_moments}). Since a conditional probability bounded by $\alpha_1/|\mathcal G|$ for every sample averages to an unconditional probability bounded by $\alpha_1/|\mathcal G|$, the two margins combine by a union bound (\Cref{app:coverage}).

The rotation $U$ and the scale $s$ are computed from the held-out sample $D_\perp$ and treated as fixed constants throughout; every probability below is over the estimation sample and the Monte Carlo noise only. This is what makes $\hat h_n$ a function of the independent $x_1,\dots,x_n$ alone, as \Cref{lem:t_md} requires.

\paragraph{The integrator.} Write $y_{b,i}:=\Pi_b x_i$ for the projected sample points, so $\hat f_n(\cdot\,;\Pi_b)$ is the equal-weight mixture of the Gaussians $\phi_s(\cdot-y_{b,i})$, $i\in[n]$. A draw from it is $y_{b,J}+W$ with $J\sim\mathrm{Unif}[n]$ and $W\sim\mathcal N(0,s^2I_k)$ independent, so
\[
    h(\hat f_n(\cdot\,;\Pi_b)) \;=\; -\frac1n\sum_{i=1}^n\E\bigl[\log_2\hat f_n(y_{b,i}+W;\Pi_b)\bigr],
\]
and we estimate each of the $n$ expectations by $n_{\mathrm{MC}}$ independent draws of $W$:
\begin{equation}\label{eq:mc_def}
    \hat h_{n,\mathrm{MC}}(\Pi_b) \;:=\; -\frac{1}{n\,n_{\mathrm{MC}}}\sum_{i=1}^n\sum_{j=1}^{n_{\mathrm{MC}}} \log_2 \hat f_n\bigl(y_{b,i} + W_{ij,b};\,\Pi_b\bigr), \qquad W_{ij,b} \iid \mathcal N(0, s^2 I_k),
\end{equation}
with the $W_{ij,b}$ independent of each other and of the sample, and $\hat h_{n,\mathrm{MC}} := N_k^{-1}\sum_b \hat h_{n,\mathrm{MC}}(\Pi_b)$. This is the integrator of \citet[eqs.~(36a)--(36b)]{goldfeld2020convergence}, which uses natural logarithms; \eqref{eq:mc_def} is their estimator divided by $\ln 2$.

\begin{lemma}[Integrator: conditional mean and variance]\label{lem:mc_moments}
Conditionally on the sample $x_{1:n}$,
\begin{enumerate}
    \item[(i)] $\E[\hat h_{n,\mathrm{MC}}(\Pi_b)\mid x_{1:n}] = h(\hat f_n(\cdot\,;\Pi_b))$ for every $b$, and hence $\E[\hat h_{n,\mathrm{MC}}\mid x_{1:n}] = \hat h_n$;
    \item[(ii)] for every $b$,
    \begin{equation}\label{eq:c_mc}
        \Var[\hat h_{n,\mathrm{MC}}(\Pi_b)\mid x_{1:n}] \;\le\; \frac{c_{\mathrm{MC}}(\Pi_b)}{(\ln 2)^2\, n\,n_{\mathrm{MC}}}\ \text{bits}^2,
    \end{equation}
    where, with $\hat m_b := \frac1n\sum_{i=1}^n \|y_{b,i}\|^2$ the mean squared norm of the component means,
    \begin{equation}\label{eq:c_mc_def}
        c_{\mathrm{MC}}(\Pi_b) \;:=\; \frac{9ks^2 + 8\bigl(2 + s\sqrt k\bigr)\hat m_b + 3\bigl(11 s\sqrt k + 1\bigr)\sqrt{\hat m_b}}{s^2};
    \end{equation}
    \item[(iii)] $\displaystyle \Var[\hat h_{n,\mathrm{MC}}\mid x_{1:n}] \;\le\; \frac{1}{(\ln 2)^2\, N_k^2\, n\,n_{\mathrm{MC}}}\sum_{b=1}^{N_k} c_{\mathrm{MC}}(\Pi_b)$ bits$^2$.
\end{enumerate}
\end{lemma}
\begin{proof}
Given the sample, $\hat f_n(\cdot\,;\Pi_b)$ is a fixed Gaussian mixture with component means $y_{b,1},\dots,y_{b,n}$, the setting of \citet[Sec.~III-D]{goldfeld2020convergence}. Their eq.~(77) is the unbiasedness of their estimator; ours is theirs divided by $\ln 2$, so (i) holds for it too. Their Theorem~5(ii) bounds the mean squared error of their estimator about $h(\hat f_n(\cdot\,;\Pi_b))$ by $c_{\mathrm{MC}}(\Pi_b)/(n\,n_{\mathrm{MC}})$ in nats$^2$, where their $m$ is the mean squared norm of the component means, here $\hat m_b$. Mean squared error is variance plus squared bias, and the bias is zero by (i), so this bounds the variance; dividing by $(\ln 2)^2$ converts it to bits$^2$ for our estimator, giving (ii). For (iii), the noise $W_{ij,b}$ is drawn independently for each block, so given the sample the $N_k$ terms $\hat h_{n,\mathrm{MC}}(\Pi_b)$ are independent; the variance of their average is $N_k^{-2}\sum_b\Var[\hat h_{n,\mathrm{MC}}(\Pi_b)\mid x_{1:n}]$, and (ii) bounds each summand.
\end{proof}

\begin{lemma}[Monte Carlo margin]\label{lem:t_mc}
Let $\sigma_{\mathrm{MC}}^2$ denote the right-hand side of \Cref{lem:mc_moments}(iii). Conditionally on the sample, for every $t>0$,
\[
    \Pr\bigl[\hat h_{n,\mathrm{MC}} > \hat h_n + t \bigm| x_{1:n}\bigr] \;\le\; \frac{\sigma_{\mathrm{MC}}^2}{\sigma_{\mathrm{MC}}^2 + t^2}.
\]
In particular, with
\begin{equation}\label{eq:t_mc}
    t_{\mathrm{MC}} \;:=\; \sigma_{\mathrm{MC}}\sqrt{|\mathcal G|/\alpha_1 - 1},
\end{equation}
the value of $t$ at which the bound equals $\alpha_1/|\mathcal G|$, we have $\hat h_{n,\mathrm{MC}} \le \hat h_n + t_{\mathrm{MC}}$ with probability at least $1-\alpha_1/|\mathcal G|$.
\end{lemma}
\begin{proof}
Given the sample, $\hat h_{n,\mathrm{MC}}$ has mean $\hat h_n$ and variance at most $\sigma_{\mathrm{MC}}^2$ by \Cref{lem:mc_moments}. Cantelli's inequality, $\Pr[X-\E X>t]\le\Var X/(\Var X+t^2)$, gives the tail bound; its right-hand side is increasing in $\Var X$, so the variance may be replaced by its upper bound. Setting it equal to $\alpha_1/|\mathcal G|$ and solving for $t$ gives \eqref{eq:t_mc}. Writing $A := \{\hat h_{n,\mathrm{MC}} > \hat h_n + t_{\mathrm{MC}}\}$, we have shown $\Pr[A\mid x_{1:n}] \le \alpha_1/|\mathcal G|$ for every realisation of the sample, and $\Pr[A] = \E[\Pr[A\mid x_{1:n}]]$, so $\Pr[A] \le \alpha_1/|\mathcal G|$ as well. Note that $t_{\mathrm{MC}}$ depends on the sample through $\hat m_b$, which is why we condition.
\end{proof}

\begin{lemma}[Sampling margin]\label{lem:t_md}
For every $t>0$,
\[
    \Pr\bigl[\hat h_n > \E[\hat h_n] + t\bigr] \;\le\; \exp\!\Bigl(-\frac{2t^2}{n\,H_b(1/n)^2}\Bigr).
\]
In particular, with
\begin{equation}\label{eq:t_md}
    t_{\mathrm{MD}} \;:=\; H_b(1/n)\sqrt{\tfrac n2\ln(|\mathcal G|/\alpha_2)},
\end{equation}
the value of $t$ at which the bound equals $\alpha_2/|\mathcal G|$, we have $\hat h_n \le \E[\hat h_n] + t_{\mathrm{MD}}$ with probability at least $1-\alpha_2/|\mathcal G|$.
\end{lemma}
\begin{proof}
The $x_i$ are independent (\Cref{asm:iid}) and changing any one of them moves $\hat h_n$ by at most $c:=H_b(1/n)$ (\Cref{cor:bdd_diff_avg}), the $\Pi_b$ and $s$ being fixed independently of the sample. McDiarmid's inequality \citep{mcdiarmid1989} gives $\Pr[\hat h_n-\E\hat h_n>t]\le\exp(-2t^2/(nc^2))$; setting this equal to $\alpha_2/|\mathcal G|$ and solving gives \eqref{eq:t_md}.
\end{proof}

\subsection{Proof of \Cref{prop:coverage}}\label{app:coverage}

\begin{proof}
Fix a pair $(k,\nu^2)\in\mathcal G$. By \Cref{lem:t_mc,lem:t_md} the events $\{\hat h_{n,\mathrm{MC}} > \hat h_n + t_{\mathrm{MC}}\}$ and $\{\hat h_n > \E[\hat h_n] + t_{\mathrm{MD}}\}$ have probability at most $\alpha_1/|\mathcal G|$ and $\alpha_2/|\mathcal G|$ respectively, so with probability at least $1-(\alpha_1+\alpha_2)/|\mathcal G|$ neither occurs, and then
\[
    \bar h_{\mathrm{LB}}(k) \;=\; \hat h_{n,\mathrm{MC}} - t_{\mathrm{MC}} - t_{\mathrm{MD}} \;\le\; \hat h_n - t_{\mathrm{MD}} \;\le\; \E\bigl[\hat h_n\bigr] \;\le\; \bar h_U(k).
\]
The last inequality is deterministic and holds for every $n$: it is the downward bias of the plug-in, $\E[\hat h_n(\Pi_b)]\le h(\tilde z^{\,\Pi_b})$ \citep[eq.~(70)]{goldfeld2020convergence}, applied to each block and averaged. A union bound over the grid gives
\[
    \Pr\Bigl[\exists\,(k,\nu^2)\in\mathcal G:\ \bar h_{\mathrm{LB}}(k) > \bar h_U(k)\Bigr] \;\le\; \sum_{(k,\nu^2)\in\mathcal G}\Bigl(\frac{\alpha_1}{|\mathcal G|} + \frac{\alpha_2}{|\mathcal G|}\Bigr) \;=\; \alpha,
\]
so \eqref{eq:h_certified} holds simultaneously for every pair with probability at least $1-\alpha$, over the sample and the Monte Carlo noise, conditionally on $D_\perp$. The union bound over $\mathcal G$ is necessary: \eqref{eq:rp_max} selects the pair $(k,\nu^2)$ that gives the largest bound, and that selection depends on the data through $\bar h_{\mathrm{LB}}(k)$. A guarantee holding at any single pre-specified pair with probability $1-\alpha$ would not cover the pair chosen afterwards, whereas \eqref{eq:h_certified} holding for every pair at once does, whichever is selected. The bound \eqref{eq:abs_certified} follows since the absolute form of \eqref{eq:rp_min} (Step~3 of the proof of \Cref{cor:rp}) is nondecreasing in the entropy, so substituting $\bar h_{\mathrm{LB}}(k)$ can only lower it; dividing by $\tr\hat\Sigma_\Psi\ge\tr\Sigma_\Psi$ and using $ks^2/\tr\hat\Sigma_\Psi=\nu^2$ gives \eqref{eq:rho_certified}.
\end{proof}

\section{Data handling}\label{app:data_handling}

The bound consumes four inputs.
\begin{itemize}
    \item The \textbf{leakage}, $\gamma$, read from the DP-SGD hyperparameters.
    \item The \textbf{block rotation}, $U$, and the \textbf{dither scale}, $s$, which we fix from a covariance estimate.
    \item The \textbf{entropy confidence bound}, $\bar h_{\mathrm{LB}}$, on the representation.
\end{itemize}

The attack experiment adds the components an adversary fits. Each has its own conditions on the data it may use, \Cref{app:dh_requirements} lists those conditions. \Cref{app:dh_partitions} turns them into labelled partitions (\Cref{tab:splits}) and \Cref{app:dh_deploy} gives the layout we recommend to a practitioner, which is also the setup we use for the privacy-spectrum experiments. \Cref{app:dh_attack} gives the layout of the attack experiment. \Cref{fig:data_flow} draws both and \Cref{alg:certify} gives the order of computation.

\subsection{What each quantity requires}\label{app:dh_requirements}

\paragraph{The protected records.} Write $D_{\mathrm{train}}$ for the records DP-SGD is run on. These are the $X$ of \Cref{sec:setting}. Two conditions follow from the assumptions.
\begin{enumerate}[label=\textbf{Req \arabic*.}, ref=Req \arabic*, resume=reqs]
    \item\label{req:iid} $D_{\mathrm{train}}$ is an i.i.d.\ sample from $P_x$ (\Cref{asm:iid}), which \Cref{lem:cmi_chain} uses to discharge $I(x_i;X^{-i})$. Any preprocessing computed across $D_{\mathrm{train}}$ and applied back to it (a dataset mean, a whitening transform, deduplication, class rebalancing, a fitted encoder) couples $x_i$ to $X^{-i}$ and breaks the argument at its first step. Per-sample maps, per-example augmentation, and any transform whose constants are fixed on data disjoint from $D_{\mathrm{train}}$ preserve it.
    \item\label{req:compose} Every trajectory computed on $D_{\mathrm{train}}$ that the adversary may observe composes into $\gamma$ (\Cref{lem:traj_chain}). A hyperparameter search of $k$ runs on the protected records is therefore charged the leakage of $k$ runs, not one; since $\gamma$ reads only $(q_t,\sigma_t)$, the search is simply one longer trajectory. Alternatively the search is run on records outside the protected set, in which case we aren't required to account its leakage. 
\end{enumerate}

\paragraph{The leakage $\gamma$.} Needs no data. \Cref{thm:gc_canonical} requires Poisson subsampling and a data-independent initialisation, together with \ref{req:compose}.

\paragraph{The rotation $U$.} We take $U = d^{-1/2}H_dV^\top$ (\Cref{rem:choice_of_U}). The Hadamard factor is a fixed matrix, so the only part of $U$ that touches data is the eigenbasis $V$ of $\Sigma$, which we estimate. \Cref{cor:rp} holds for every orthogonal $U$, so $U$ may be computed from any data, but \Cref{lem:bdd_diff} requires the $\Pi_b$ to be fixed independently of the estimation sample, so:
\begin{enumerate}[label=\textbf{Req \arabic*.}, ref=Req \arabic*, resume=reqs]
    \item\label{req:U} the sample on which $V$ is estimated is disjoint from the estimation sample.
\end{enumerate}
$V$ buys tightness, not validity: every orthogonal $U$ is admissible (\Cref{cor:rp}), so where no such sample can be spared one takes $V = I_d$, giving $U = d^{-1/2}H_d$ requiring no data. We use the data-free variant $U=d^{-1/2}H_dD$, with $D$ a fixed random diagonal sign flip, for the ImageNette pixel bound and for every LPIPS bound (\Cref{rem:choice_of_U}).

\paragraph{The dither scale and the reference scale.} The estimator convolves the sample with a Gaussian of bandwidth $s=\nu\sqrt{\tr\hat\Sigma/k}$, so $s$ is set by an estimate of $\tr\Sigma$.
\begin{enumerate}[label=\textbf{Req \arabic*.}, ref=Req \arabic*, resume=reqs]
    \item\label{req:trace} The sample on which $\tr\Sigma$ is estimated is disjoint from the estimation sample. \Cref{lem:t_md} treats $s$ as a constant and lets one record of the estimation sample move; if $s$ were computed from that sample, moving one record would also move the bandwidth, and the bounded-differences constant of \Cref{lem:bdd_diff} no longer applies.
\end{enumerate}
The same estimate is the reference scale $\rmseprior^2 = \tr\Sigma$ of $\rho$. The guarantee is the absolute bound \eqref{eq:master_abs}, in which $\tr\Sigma$ does not appear and the ratio only rescales it - so a point estimate suffices. The attack-bound comparison of \Cref{fig:tightness,fig:tightness_mnist} is unaffected in any case, both sides being divided by the same number. A practitioner who wants a confidence bound on $\rho^\star$, and so provably at most one, divides instead by a one-sided upper confidence bound on $\tr\Sigma$ (for example using empirical-Bernstein (\Cref{app:exp_entropy})).

\paragraph{The entropy bound $\bar h_{\mathrm{LB}}$.} $\bar h_U(k)$ is a property of the source $P_x$, and \Cref{prop:coverage} lower-bounds it from any i.i.d.\ sample of $P_x$ on which $U$ and $s$ were not computed.
\begin{enumerate}[label=\textbf{Req \arabic*.}, ref=Req \arabic*, resume=reqs]
\item\label{req:est} The estimation sample is i.i.d.\ from $P_x$ and disjoint from the sample of \ref{req:U}. Both margins scale as $n^{-1/2}$ (\Cref{app:margins}), so the sample should be as large as possible.
\end{enumerate}
Nothing requires the estimation sample to be disjoint from $D_{\mathrm{train}}$. Under \ref{req:iid} the training set is itself an i.i.d.\ sample of $P_x$, and using it estimates the entropy of exactly the source the protected records are drawn from, with the most data. Coverage is then over the random draw of the training set, not conditional on the realised one. Any other sample must be a draw from $P_x$ as well, since \Cref{prop:coverage} bounds the entropy of whatever distribution the sample comes from.

\paragraph{Validation and test.} Neither is an input to any estimand.
\begin{enumerate}[label=\textbf{Req \arabic*.}, ref=Req \arabic*, resume=reqs]
    \item\label{req:val} As is ordinarily practiced, one can't use the same data to both train a model and evaluate its performance.
    \item\label{req:test} Records on which accuracy is reported are disjoint from both data used to train the model and data used to evaluate the model in order to for example pick hyperparameters. 
\end{enumerate}
Neither set is trained on, so neither carries $\gamma$ and the bound makes no claim about it. 

\paragraph{The attack.} The bound quantifies over reconstructions $g(\theta,X^{-i})$ with internal randomness and knowledge of $P_x$ (\Cref{sec:setting}); an experiment testing it must keep the attack inside that class.
\begin{enumerate}[label=\textbf{Req \arabic*.}, ref=Req \arabic*, resume=reqs]
    \item\label{req:shadow} Everything the attack fits or draws from --- here the candidate set and the centring mean --- excludes the target. An attack that has seen its target is not a function of $(\theta,X^{-i})$ and can fall below the bound without the bound being wrong. $X^{-i}$ itself is known to the adversary, so the shadow pool may include it.
    \item\label{req:precoder} The attacked architecture, including the centring mean of its input (\Cref{app:exp_attack}), is fixed on records disjoint from the attacked set $X=X^{-i}\cup\{x_i\}$. The bound holds for every architecture at a given $(q,\sigma,T)$, so a fixed input map is on the same footing as the linear layer itself; but one fit on the attacked records is preprocessing across $X$ and violates \ref{req:iid}.    
    \item\label{req:targets} The targets are i.i.d.\ from $P_x$, so that the mean error over them estimates the expectation the bound controls, and are disjoint from the sample $U$ is computed on (\ref{req:U}, admissibility) and from the shadow pool (\ref{req:shadow}). They need not be disjoint from the estimation sample. \Cref{prop:coverage} is a statement about the population entropy $\bar h_U(k)$, holding on an event of probability $1-\alpha$ over the estimation sample. On that event the bound holds for the expected error of every $g$, whichever records the targets are. The attack's mean error over the targets is an estimate of that expectation, and its expectation does not depend on whether the targets also appear in the estimation sample. Overlap makes the two random quantities dependent, but changes neither the probability of the coverage event nor the expected attack error, so the comparison in \Cref{fig:tightness,fig:tightness_mnist} is valid either way.
\end{enumerate}

\subsection{The partitions}\label{app:dh_partitions}

\Cref{tab:splits} names a partition for each role above and records, for each, what it supplies, the requirement that constrains it, which partitions it must be disjoint from, and the property its size must have. A $\star$ marks a disjointness the guarantee or the threat model requires. The rest are kept for a clean joint statement and could be relaxed. Two partitions are shared by both experiments: $D_{\perp}$, which fixes $U$ and $s$, and $D_{\mathrm{est}}$, which supplies the entropy bound.

\begin{table}[t]
\centering\footnotesize
\setlength{\tabcolsep}{3pt}
\renewcommand{\arraystretch}{1.3}
\caption{The partitions of \Cref{fig:data_flow}, and what constrains each. The requirement column refers to \Cref{app:dh_requirements}; $\star$ marks a disjointness the guarantee or the threat model requires, the rest are kept for a clean joint statement and could be relaxed.}
\label{tab:splits}
\begin{tabular}{@{}l >{\raggedright\arraybackslash}p{34mm} l >{\raggedright\arraybackslash}p{28mm} >{\raggedright\arraybackslash}p{32mm}@{}}
\hline
Split & Supplies & Requirement & Disjoint from & Size and property \\
\hline\hline
\multicolumn{5}{@{}l}{\emph{shared --- the source side, computed once per dataset}}\\
\hline
$D_{\perp}$ & $U=d^{-1/2}H_dV^\top$, $\tr\hat\Sigma$, dither scales & \ref{req:U}, \ref{req:trace} & $D_{\mathrm{est}}^{\star}$, $D_{\mathrm{targets}}^{\star}$ & $\gtrsim d$, for a full-rank $\hat\Sigma$; may be replaced by a data-free $U$ and a declared $\tr\Sigma$ \\
\hline
$D_{\mathrm{est}}$ & $\bar h_{\mathrm{LB}}$ & \ref{req:est} & $D_{\perp}^{\star}$ & largest available; i.i.d.\ from $P_x$; margins $\propto n^{-1/2}$ \\
\hline\hline
\multicolumn{5}{@{}l}{\emph{(a) deployment, and the privacy-spectrum runs}}\\
\hline
$D_{\mathrm{train}}$ & every charged DP-SGD run & \ref{req:iid}, \ref{req:compose} & $D_{\mathrm{val}}$, $D_{\mathrm{test}}$ & i.i.d.\ from $P_x$; no cross-record preprocessing; may overlap $D_\perp$ and $D_{\mathrm{est}}$ \\
\hline
$D_{\mathrm{val}}$ & selection among runs & \ref{req:val} & $D_{\mathrm{train}}$ & small \\
\hline
$D_{\mathrm{test}}$ & reported accuracy & \ref{req:test} & $D_{\mathrm{train}}$, $D_{\mathrm{val}}$ & small \\
\hline\hline
\multicolumn{5}{@{}l}{\emph{(b) the attack experiment}}\\
\hline
$D_{\mathrm{shadow}}$ & candidate set, centring mean & \ref{req:shadow}, \ref{req:precoder} & $D_{\mathrm{targets}}^{\star}$ & most of $D_{\mathrm{est}}$ \\
\hline
$D_{\mathrm{targets}}$ & the reconstructed records & \ref{req:targets} & $D_{\perp}^{\star}$, $D_{\mathrm{shadow}}^{\star}$ & $100$; i.i.d.\ from $P_x$; enough for the standard error of the mean; may lie in $D_{\mathrm{est}}$ \\
\hline
\end{tabular}
\end{table}

\subsection{Deployment, and the privacy-spectrum runs}\label{app:dh_deploy}

A practitioner makes two splits of the available data, independently of each other. For the bound, they draw a sample $D_{\perp}$ of at least $d$ records, on which they compute the covariance $\hat\Sigma$, whose eigenvectors give $V$ and whose trace gives $\tr\hat\Sigma$, and a disjoint estimation sample $D_{\mathrm{est}}$, as large as possible. For training, they split the data into $D_{\mathrm{train}}$, $D_{\mathrm{val}}$ and $D_{\mathrm{test}}$ as in any training pipeline. Only $D_\perp$ and $D_{\mathrm{est}}$ must be disjoint (\ref{req:U}, \ref{req:trace}). $D_{\mathrm{val}}$ is used to evaluate performance during training for purposes such as hyperparameter selection and $D_{\mathrm{test}}$ is used to report final utility achieved. The hyperparameter search is either charged through $\gamma$ (by summing \Cref{thm:gc_canonical} over every run), moved onto records outside the protected set, or done privately (such as on trusted private compute). The learning rate and architecture do not enter $\gamma$, so a search over them at fixed $(q_t,\sigma_t)$ adds one trajectory's leakage per run and nothing more. This layout is also that of the privacy-spectrum runs (\Cref{fig:spectrum_all}), so \Cref{fig:data_flow}(a) is both the recommendation and what we do: $D_\perp$ and $D_{\mathrm{est}}$ are the splits of the attack experiment (\Cref{app:dh_attack}), drawn independently of the utility runs' split into $D_{\mathrm{train}}$, $D_{\mathrm{val}}$ and $D_{\mathrm{test}}$. As is common in the literature \citep{ponomareva2023dp}, and to keep the experimental results simple to interpret, we assume hyperparameter selection leaks nothing. Where no $D_{\perp}$ can be spared, $U$ is taken data-free and $\tr\Sigma$ is replaced by a declared upper bound, at a cost in tightness alone.\footnote{Any declared value keeps the absolute bound \eqref{eq:master_abs} valid, since the dither scale it sets may be any positive number; the ratio bound stays valid provided the declared value is an upper bound on $\tr\Sigma$, since dividing by a larger number can only lower it. For $8$-bit pixels, $\tr\Sigma\le d\cdot255^2/4$ always.}

\subsection{The attack experiment}\label{app:dh_attack}

The attack experiment reuses the rotation, the dither scales and the entropy bounds of (a) unchanged, so the two experiments differ only in the leakage their mechanism incurs. It uses two further sets. $D_{\mathrm{shadow}}\subseteq D_{\mathrm{est}}$ is the attack's candidate set and supplies its centring mean (\ref{req:shadow}, \ref{req:precoder}). $D_{\mathrm{targets}}$ is $100$ records from a held-out slice disjoint from $D_{\mathrm{shadow}}$, $D_{\mathrm{est}}$ and $D_\perp$ (\ref{req:targets}).

\begin{figure}[t]
\centering
\begin{tikzpicture}[
  font=\small, scale=1, every node/.style={transform shape},
  pool/.style={draw, fill=black!5,  minimum height=5.5mm, align=center, inner sep=1.5pt},
  bnd/.style ={draw, fill=green!12, minimum height=5.5mm, align=center, inner sep=1.5pt},
  prt/.style ={draw, fill=red!8,    minimum height=5.5mm, align=center, inner sep=1.5pt},
  atk/.style ={draw, fill=blue!8,   minimum height=5.5mm, align=center, inner sep=1.5pt},
  hd/.style={font=\small\bfseries, anchor=west}
]
\node[hd] at (0,1.5) {(a) Deployment, and the privacy-spectrum runs of \Cref{fig:spectrum_all}};

\node[pool, minimum width=132mm, anchor=west] (poolA) at (0,1.0) {available data, i.i.d.\ $P_x$};

\node[bnd, minimum width=16mm, anchor=west] (Aperp) at (0,0.35)     {$D_{\perp}$};
\node[bnd, minimum width=114mm, anchor=west] (Aest) at (18mm,0.35)  {$D_{\mathrm{est}}$};
\node[prt, minimum width=96mm, anchor=west] (Atr)   at (0,-0.3)     {$D_{\mathrm{train}}$};
\node[pool,minimum width=16mm, anchor=west] (Ava)   at (98mm,-0.3)  {$D_{\mathrm{val}}$};
\node[pool,minimum width=16mm, anchor=west] (Ate)   at (116mm,-0.3) {$D_{\mathrm{test}}$};

\begin{scope}[yshift=-30mm] 
\node[hd] at (0,1.85) {(b) The attack experiment of \Cref{fig:tightness}};
\draw[decorate, decoration={brace, amplitude=3pt}, black!55]
  (18mm,1.15) -- (94mm,1.15)
  node[midway, above=2pt, font=\scriptsize\itshape] {$D_{\mathrm{est}}$, unchanged from (a): entropy estimated on all of it};
\node[bnd, minimum width=16mm, anchor=west] (Bperp) at (0,0.5)     {$D_{\perp}$};
\node[atk, minimum width=76mm, anchor=west] (Bsh)   at (18mm,0.5)  {$D_{\mathrm{shadow}}\subseteq D_{\mathrm{est}}$};
\node[prt, minimum width=18mm, anchor=west] (Btg)   at (96mm,0.5)  {$D_{\mathrm{targets}}$};
\end{scope}
\end{tikzpicture}
\caption{How the available data is partitioned. Green blocks supply the bound, red blocks are the protected records trained on, and blue blocks are data the attack uses. In \textbf{(a)} the bound's split into $D_\perp$ and $D_{\mathrm{est}}$ (top row) and the training split (bottom row) are drawn independently and may overlap; only $D_\perp$ and $D_{\mathrm{est}}$ must be disjoint. $D_{\perp}$ and $D_{\mathrm{est}}$ are shared, so the rotation and the entropy estimate are the same in the deployment \textbf{(a)} and the attack experiment \textbf{(b)}; the targets are held out from both.}
\label{fig:data_flow}
\end{figure}

\begin{algorithm}[t]
\caption{Reconstruction lower bound for a DP-SGD deployment.}
\label{alg:certify}
\begin{algorithmic}[1]
\Require $(q_t,\sigma_t)_{t=1}^T$ for every run charged to $D_{\mathrm{train}}$; samples $D_{\perp}$ and $D_{\mathrm{est}}$ (i.i.d.\ from $P_x$; $D_{\mathrm{est}}$ may be $D_{\mathrm{train}}$); grid $\mathcal G$; confidence $\alpha$
\State $\gamma \gets \tfrac{1}{\ln 2}\sum_{\text{runs}}\sum_{t=1}^{T} K_t^{\mathrm{PRV}}$ \Comment{\Cref{thm:gc_canonical}; \ref{req:compose}}
\State On $D_{\perp}$: fix, for each $\Psi$, $U = m^{-1/2}H_mV^{\top}$, the dither scales $s_{k,\nu}$, and $\tr\hat\Sigma_\Psi$ \Comment{\ref{req:U}, \ref{req:trace}}
\For{$\Psi\in\{\mathrm{id},\lpipsmap\}$ and $(k,\nu^2)\in\mathcal G$}
  \State On $D_{\mathrm{est}}$: $\bar h_{\mathrm{LB}}(k) \gets \hat h_{n,\mathrm{MC}} - t_{\mathrm{MC}} - t_{\mathrm{MD}}$ \Comment{\Cref{prop:coverage}; \ref{req:est}}
\EndFor
\State $\rho^\star_\Psi \gets$ \eqref{eq:rho_certified}
\State Train by Poisson-subsampled DP-SGD on $D_{\mathrm{train}}$; to target a given $\rho^\star$, invert steps 1 and 6 for $\sigma$
\end{algorithmic}
\end{algorithm}

\section{Experimental details}\label{app:exp}

\Cref{app:exp_data,app:exp_entropy,app:exp_phi} apply to both experiments; \Cref{app:exp_spectrum} to \Cref{fig:spectrum_headline,fig:spectrum_all}; \Cref{app:exp_attack,app:exp_compressor,app:exp_extra} to \Cref{fig:tightness,fig:tightness_mnist}.

\subsection{Datasets}\label{app:exp_data}

\begin{center}\small
\begin{tabular}{lcccccc}
\toprule
Dataset & $C\!\times\!H\!\times\!W$ & $d$ & $\tr\Sigma$ & $\rmseprior/\sqrt d$ & $n_{\mathrm{est}}$ & $|D_\perp|$ \\
\midrule
MNIST & $1\!\times\!28\!\times\!28$ & $784$ & $3.43\!\times\!10^{6}$ & $66.2$ & $59$k & $10$k \\
CIFAR-10 & $3\!\times\!32\!\times\!32$ & $3072$ & $1.24\!\times\!10^{7}$ & $63.6$ & $49$k & $10$k \\
CelebA & $3\!\times\!64\!\times\!64$ & $12{,}288$ & $5.91\!\times\!10^{7}$ & $69.3$ & $44$k & $117$k \\
ImageNette & $3\!\times\!128\!\times\!128$ & $49{,}152$ & $2.46\!\times\!10^{8}$ & $70.7$ & $10$k & $1$k \\
\bottomrule
\end{tabular}
\end{center}

The deployment runs span $d=784$ to $49{,}152$, greyscale and color. The CelebA dataset serves as a particularly interesting dataset as recovery where one is able to ascertain the identity of the individual in the samples constitutes a natural privacy threat that one can validate with their own eyes. We deploy our bound and evaluate the utility achieved by DP-SGD across all the datasets above. For ImageNette, $D_\perp$ supplies only $\tr\hat\Sigma$; no eigenbasis is estimated on it (\Cref{rem:choice_of_U}).

The attack experiment uses MNIST and CelebA, a low-dimensional greyscale source and a colour face source. On CelebA, recovering a recognisable face is the natural privacy harm, so it is the case in which the perceptual bound is the one that matters. Each configuration is attacked on the same $100$ held-out targets at every leakage (\Cref{app:dh_attack}).

CelebA \citep{liu2015deep} uses the binary \emph{Male} attribute as its label. Networks consume images rescaled to $[0,1]$ (CIFAR-10 utility training additionally standardises by fixed, public per-channel constants); entropy is computed on the raw 8-bit grid, so \Cref{asm:discrete} holds with $b_x=8$, $c=256$ and $\Delta=1$. No preprocessing couples samples: every normalisation uses constants fixed independently of $X$, so \Cref{asm:iid} holds (\Cref{app:data_handling}).

\subsection{The entropy bound}\label{app:exp_entropy}

The splits are those of \Cref{app:dh_deploy}, with sizes in \Cref{app:exp_data}.

\paragraph{Reference scale.} We normalise $\rho$ by the point estimate $\tr\hat\Sigma_\Psi$ from $D_\perp$ (\Cref{tab:trsigma}). This does not affect the guarantee: the absolute bound \eqref{eq:master_abs} does not involve $\tr\Sigma_\Psi$, and any other normaliser $c$ simply rescales $\rho^\star$ by $(\tr\hat\Sigma_\Psi/c)^{1/2}$. To gauge the estimation error we also give a $95\%$ upper confidence bound from the empirical Bernstein inequality of \citet{maurer2009empirical}, which requires i.i.d.\ variables with mean $\tr\Sigma_\Psi$ and a range fixed in advance. Pairing the records of $D_\perp$ at random supplies these: for independent $x,x'\sim P_x$, $Y:=\tfrac12\|\Psi(x)-\Psi(x')\|^2$ has mean exactly $\tr\Sigma_\Psi$ and, unlike $\|\Psi(x)-\E \Psi(x)\|^2$, needs no estimate of the mean. Its range is $[0,B/2]$, with $B$ the largest possible value of $\|\Psi(x)-\Psi(x')\|^2$. In pixel space $B=255^2d$. In LPIPS space this squared distance is the LPIPS distance, a sum over layers $\ell$ of the spatial average of $\sum_c w_{\ell c}(a_c-b_c)^2$, where $w_{\ell c}\ge0$ is the learned LPIPS weight on channel $c$ of layer $\ell$ and $a,b$ are the two images' normalised features at one location. These features are nonnegative with norm at most one, so $\|a-b\|^2\le2$, each layer contributes at most $2\max_c w_{\ell c}$, and $B=2\sum_\ell\max_c w_{\ell c}$. Empirical Bernstein is much tighter here than Hoeffding, whose width scales as $B/\sqrt m$ for $m$ pairs: $B$ is attained only by two opposite-extreme records, so $Y$ varies far less than its range, and empirical Bernstein uses the sample variance of $Y$ and pays for $B$ only at rate $B/m$. The resulting bound exceeds the point estimate by at most $5.8\%$ on MNIST, CIFAR-10 and CelebA, and by $16.6\%$ on ImageNette, where $|D_\perp|$ is small. Normalising by it instead would lower every reported bound by at most $0.026$, or $0.064$ on ImageNette, and the bound would then hold with probability at least $90\%$ by a union bound with \Cref{prop:coverage}.

\paragraph{Grid.} We take $\mathcal G=K\times N$, with block sizes $K$ the divisors of $d$ up to $32$ in pixel space ($\{1,2,4,7,8,14,16,28\}$ on MNIST, $\{1,2,3,4,6,8,12,16,24,32\}$ on the others) and $K=\{1,2,4,8,16,32\}$ in LPIPS space, and dither levels $N=\{0.2,0.4,0.8,1.6,3.2,6.4\}$, so $|\mathcal G|=48$ (MNIST pixel), $60$ (other pixel) and $36$ (LPIPS), both fixed before any entropy is estimated, as \Cref{app:coverage} requires. In pixel space $K$ contains only divisors of $d$, and in LPIPS space only divisors of the padded dimension $32{,}768$ (see Blocks below). The selected pair often lies on the boundary of the grid (\Cref{tab:spectrum_full}): $k^\star=\max K$ at $\varepsilon\ge8$ in pixel space on MNIST, CIFAR-10 and ImageNette (and at $\varepsilon=16$ on CelebA), and at almost every leakage in LPIPS space; $\nu^{2\star}$ reaches an end of $N$ in a few places (CelebA pixel: $0.2$ at $\varepsilon\le1$ and $\varepsilon=32$; CelebA LPIPS: $6.4$ at $\varepsilon\le0.5$; MNIST LPIPS: $6.4$ at $\varepsilon=2$).

\paragraph{Blocks.} We never store the rotation as a dense matrix. Instead we keep the Hadamard factor as a Kronecker product of small Hadamard matrices, so applying $U$ costs one small matrix product per factor. In pixel space we use $H_{784}=H_{28}\otimes H_{28}$, $H_{3072}=H_{12}\otimes H_{64}\otimes H_4$, $H_{12288}=H_{12}\otimes H_{64}\otimes H_{16}$ and $H_{49152}=H_{12}\otimes H_{64}\otimes H_{64}$, with the power-of-two factors from Sylvester's doubling and the others from Paley's construction. These are Hadamard matrices of order exactly $d$, so no zero-padding is needed; padding adds dimensions that carry no entropy and loosens the bound. In LPIPS space we zero-pad the $m=31{,}872$ features to $32{,}768=2^{15}$ and rotate by $U=32768^{-1/2}H_{32768}D$, with $H_{32768}=H_{64}\otimes H_{64}\otimes H_8$ and $D$ a fixed random diagonal sign flip, so no eigenbasis is estimated. The padded coordinates are deterministically zero, so by \Cref{rem:choice_of_U} the bound transfers to the unpadded features, at a cost in tightness set by the padding fraction ($2.7\%$), and the block variances equal $\tr\Sigma_{\lpipsmap}$ only on average over the signs.

\paragraph{Estimator and margins.} The plug-in evaluates $\log_2\hat f_n$ at each Monte Carlo point exactly, summing over all $n_{\mathrm{est}}$ centres with a numerically stable log-sum-exp. We draw $n_{\mathrm{MC}}=2$ dither replicates per centre at $k=1$ and $n_{\mathrm{MC}}=\max(k,4)$ at block size $k\ge2$. We set $\alpha=0.05$, split it evenly between the two margins, and divide each by $|\mathcal G|$, separately for each representation. Each bound therefore holds at confidence $1-\alpha$ on its own, and the pixel and LPIPS bounds hold jointly at $1-2\alpha$. Both margins are in bits of the block entropy $\bar h_U(k)$, the entropy of one $k$-dimensional block averaged over the blocks, so neither grows with $d$. Their cost to the final bound is the gap between $\rho^\star$ and its point-estimate version (thin dashed in \Cref{fig:tightness,fig:tightness_mnist}), which drops both margins and maximises over $\mathcal G$ again: at most $0.075$ in pixel space and $0.047$ in LPIPS space at $\varepsilon=1$ ($0.083$ in pixel space on ImageNette), and $0.01$--$0.035$ at $\varepsilon=8$ and $16$ (up to $0.046$ at $\varepsilon=32$, where $\rho^\star$ itself is only $0.13$--$0.19$). The selected pairs $(k^\star,\nu^{2\star})$ are listed in \Cref{tab:spectrum_full}.

\begin{table}[htbp]
    \centering\small\setlength{\tabcolsep}{4pt}
    \caption{The reference scale and how far its finite-sample one-sided $95\%$ upper confidence bound (\Cref{app:exp_entropy}) lies above it, relative to the point estimate, estimated on $D_\perp$ (sizes in \Cref{app:exp_data}). $\rmseprior/\sqrt d$ is the per-coordinate root-mean-squared error of the prior-mean guess, in grey levels for pixels and feature units for $\lpipsmap$ ($m=31{,}872$).}
    \label{tab:trsigma}
    \begin{tabular}{l r r r r r r}
        \toprule
        & \multicolumn{3}{c}{pixel ($\Psi=\mathrm{id}$)} & \multicolumn{3}{c}{perceptual ($\Psi=\lpipsmap$)} \\
        \cmidrule(lr){2-4}\cmidrule(lr){5-7}
        Dataset & $\tr\Sigma$ & UCB margin & $\rmseprior/\sqrt d$ & $\tr\Sigma_{\lpipsmap}$ & UCB margin & $\rmseprior/\sqrt m$ \\
        \midrule
        MNIST       & $3.43\times10^{6}$ & $2.5\%$ & $66.2$ & $0.1338$ & $5.8\%$ & $2.05\times10^{-3}$ \\
        CIFAR-10    & $1.24\times10^{7}$ & $3.2\%$ & $63.6$ & $0.2463$ & $3.3\%$ & $2.78\times10^{-3}$ \\
        CelebA      & $5.91\times10^{7}$ & $0.5\%$ & $69.3$ & $0.1445$ & $0.6\%$ & $2.13\times10^{-3}$ \\
        ImageNette  & $2.46\times10^{8}$ & $16.6\%$ & $70.7$ & --- & --- & --- \\
        \bottomrule
    \end{tabular}
\end{table}

\subsection{The LPIPS representation}\label{app:exp_phi}

$\lpipsmap$ is the frozen LPIPS network \citep{zhang2018unreasonable} with its AlexNet backbone \citep{krizhevsky2012imagenet}. We concatenate its normalised feature maps, each rescaled by the square root of the LPIPS weights and the spatial size, so that $\|\lpipsmap(x)-\lpipsmap(x')\|^2$ is the LPIPS distance, giving $m=31{,}872$ coordinates. Images are resized to $64\times64$ first, so the perceptual bound measures LPIPS at that resolution.

$\lpipsmap$ is a fixed public map, so it is admissible (\Cref{def:dither}) provided its training images are disjoint from the records we protect. We do not report an LPIPS bound on ImageNette, whose images are drawn from the ImageNet training set on which the LPIPS features were trained.

\paragraph{The data-blind baseline.} The prior mean $\E[\lpipsmap(x)]$ is the best data-blind guess in feature space, but it is not the feature vector of any image, so an attacker who must output an image cannot return it. For any image $\hat x$ chosen without seeing the release,
\[
    \E\|\lpipsmap(x)-\lpipsmap(\hat x)\|^2 \;=\; \tr\Sigma_{\lpipsmap} + \|\lpipsmap(\hat x)-\E\lpipsmap(x)\|^2 .
\]
The first term does not depend on $\hat x$, so the optimal data-blind image is the one whose features lie closest to the prior mean. We find it by searching over images directly. Because the objective is not convex, we run the search from $16$ different starting images, drawn from $D_{\mathrm{shadow}}$, and keep the best result. From each, we run $400$ steps of Adam on the pixels of $\hat x$ to minimise $\|\lpipsmap(\hat x)-\hat\mu_{\lpipsmap}\|^2$ plus a total-variation penalty of weight $10^{-4}$, and we keep the final image whose features are closest to $\hat\mu_{\lpipsmap}$. Here $\hat\mu_{\lpipsmap}$ is the mean feature vector over $2{,}000$ images of $D_{\mathrm{shadow}}$, and the pixels are kept in the valid range by the reparametrisation $\hat x=\tfrac12(\tanh u+1)$. The gradient passes through the frozen network, so each step moves the image towards features closer to the mean. Descent may stop at a local minimum, so the error we report may slightly exceed that of the optimal data-blind image.

Because the second term is positive for every image, this baseline has error above $\tr\Sigma_{\lpipsmap}$. We normalise $\rhostarlp$ by $\sqrt{\tr\Sigma_{\lpipsmap}}$, as in \eqref{eq:rho_def}, which keeps $\rho^\star\le1$. On that scale the image baseline sits slightly above $1$ ($1.131$ on MNIST and $1.045$ on CelebA), and we draw it as the grey dashed line in the LPIPS panels of \Cref{fig:tightness,fig:tightness_mnist}.

\subsection{The privacy-spectrum experiment}\label{app:exp_spectrum}

\paragraph{The bound.} We report results against $\varepsilon$, the quantity practitioners compare, rather than $\sigma$ or $\gamma$. For each target $\varepsilon$ we calibrate the noise multiplier $\sigma$ with the PRV accountant \citep{gopi2021numerical,yousefpour2021opacus} at $\delta=n_{\mathrm{train}}^{-1.1}$, with $n_{\mathrm{train}}$ the training-set size. Since $q$ and $\sigma$ are constant, every step has the same charge $K^{\mathrm{PRV}}$ in \Cref{thm:gc_canonical}, so $\gamma=TK^{\mathrm{PRV}}/\ln2$ needs a single one-dimensional integral per configuration, which we evaluate with a $256$-node quadrature rule. Against a reference adaptive quadrature its error is at most $6.7\times10^{-9}$ bits in $\gamma$ over all configurations, six orders of magnitude below the smallest entropy margin per dimension ($0.0054$ bits), so it has no visible effect on $\rho^\star$.

\paragraph{Configuration.} Our configuration follows common DP-SGD practice \citep{ponomareva2023dp,de2022unlocking} in its main ingredients: large Poisson-subsampled batches trained for many steps, GroupNorm, augmentation multiplicity and weight averaging (below). Hyperparameters are not tuned per privacy level: expected batch size $b=1024$ for $100$ epochs, giving $T=4{,}900$ steps at $q=0.0204$ (MNIST), $T=4{,}400$ at $q=0.0227$ (CIFAR-10, CelebA) and $T=1{,}000$ at $q=0.1$ (ImageNette), with a data-independent initialisation as \Cref{lem:traj_chain} requires. We clip at $C=1$ on every dataset; neither $\gamma$ nor $\rho^\star$ depends on $C$ (\Cref{sec:bound_gamma}).

\paragraph{Utility.} We train a LeNet-5 \citep{lecun1998gradient} on MNIST (${\sim}62$K parameters), a Wide-ResNet-16-4 on CIFAR-10 (${\sim}2.7$M) and a ResNet-18 on CelebA and ImageNette (${\sim}11$M). All use GroupNorm \citep{wu2018group} rather than BatchNorm, which would make each record's gradient depend on the others in its batch and break \Cref{lem:per_step_channel}. Training uses Opacus \citep{yousefpour2021opacus} at the configuration above, with SGD at momentum $0.9$ and cross-entropy loss. We use augmentation multiplicity $4$ on CIFAR-10, CelebA and ImageNette ($1$, i.e.\ none, on MNIST) \citep{de2022unlocking,ponomareva2023dp}, meaning that four augmentations of each example are drawn and their gradients averaged before clipping, so every record still contributes a single clipped gradient. We also keep an exponential moving average of the iterates during training (decay $0.999$) and use those averaged weights when measuring accuracy, while training itself continues from the raw iterates \citep{de2022unlocking,ponomareva2023dp}. Averaging cancels much of the noise DP-SGD injects, since that noise is independent across steps, and it is post-processing of released iterates, so it costs no privacy. The learning rate is fixed per dataset ($0.5$ on MNIST, $1.0$ on CIFAR-10, $0.25$ on CelebA and ImageNette) and reused at every privacy level, so the accuracies we report are attainable but not tuned per level. Every privacy level is trained with $3$ seeds and we report the mean and standard deviation of test accuracy. The non-private baseline is the same architecture and optimiser trained without clipping or noise, with a single seed and its own learning rate, selected on $D_{\mathrm{val}}$ ($0.05$ on MNIST, CIFAR-10 and ImageNette, $0.1$ on CelebA), and accuracy is reported on $D_{\mathrm{test}}$.

\paragraph{Hyperparameter tuning.} Any search over the clip norm or the learning rate reads the training data, so strictly their leakage composes with that of the final run (\Cref{lem:traj_chain}) (assuming that an adversary can observe the noisy trajectories of these training runs). As is standard in the DP-SGD literature \citep{ponomareva2023dp}, we do not include it, and the reported $\varepsilon$ and $\rho^\star$ are those of the final run alone. Including it requires no new theory. Since $\gamma$ reads only the schedule $(q_t,\sigma_t)$, a search of $k$ runs is simply a longer trajectory and adds the leakage of $k$ runs, so a practitioner can either sum \Cref{thm:gc_canonical} over every run or move the search onto data outside the protected set (\Cref{app:dh_deploy}).

\begin{figure}[htbp]
    \centering
    \includegraphics[width=\textwidth]{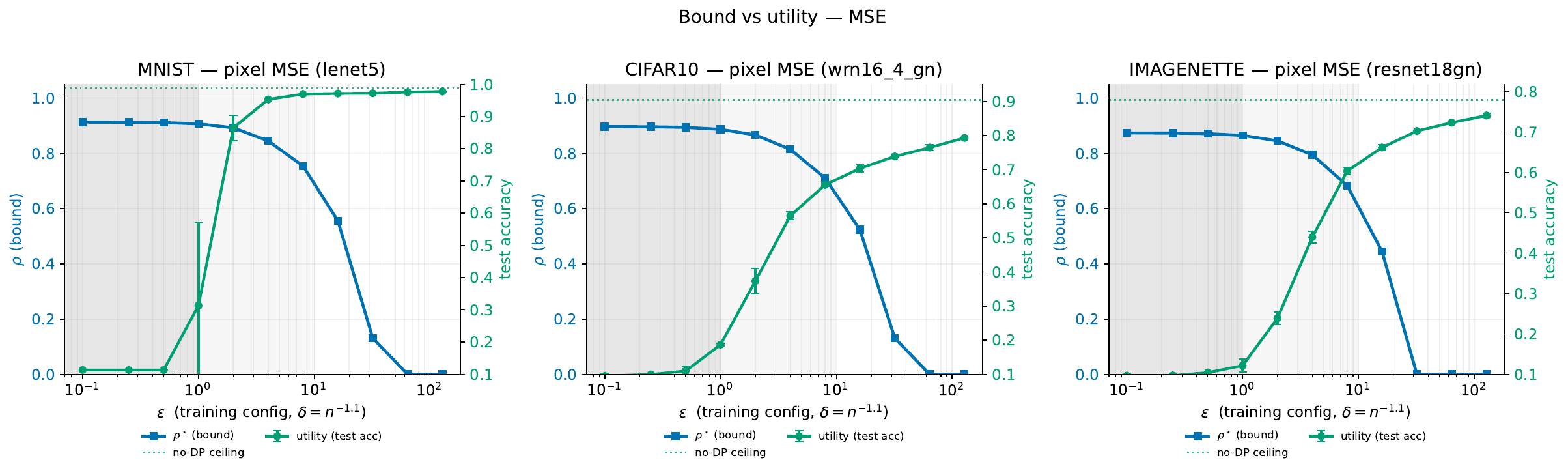}
    \includegraphics[width=0.66\textwidth]{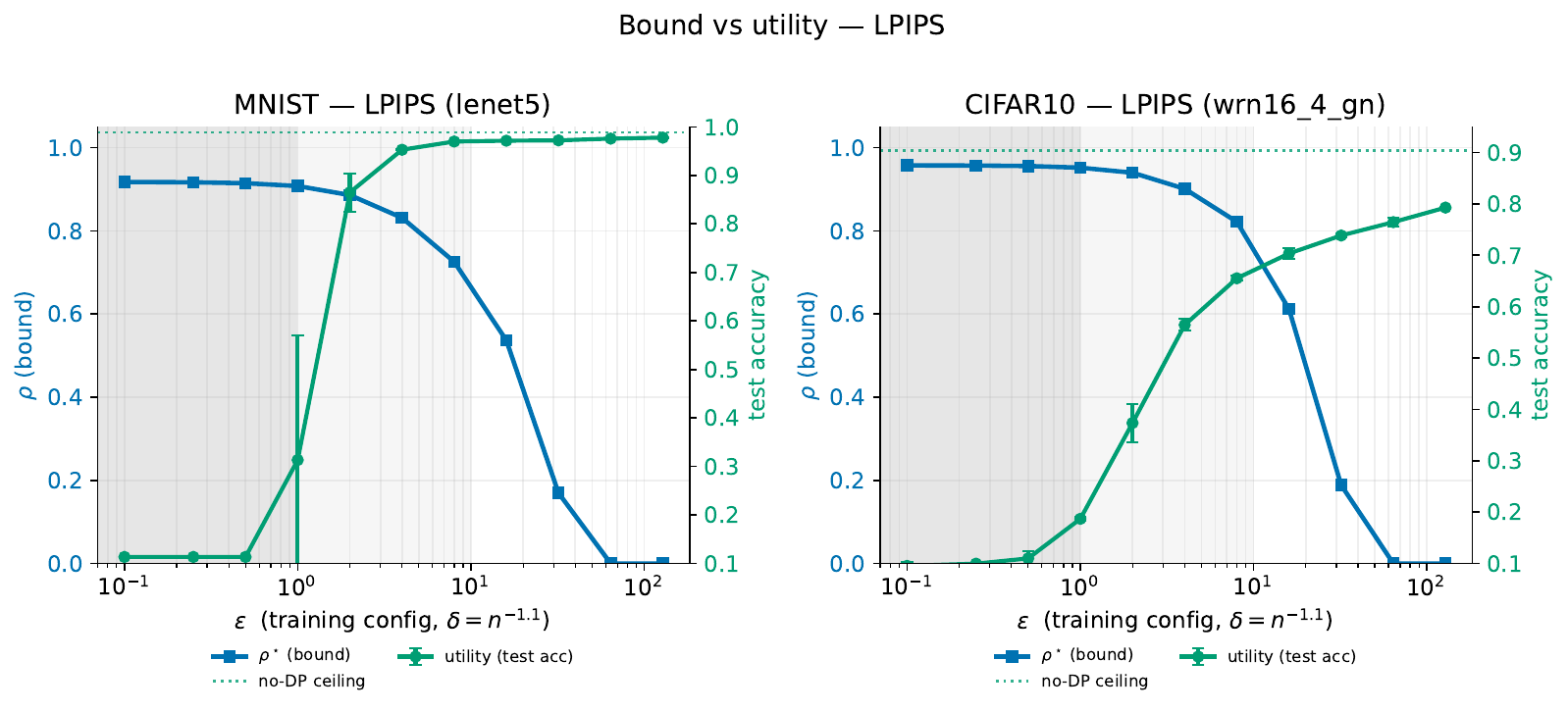}
    \caption{What a deployed DP-SGD run guarantees across the privacy spectrum, on MNIST, CIFAR-10 and ImageNette (CelebA is in \Cref{fig:spectrum_headline}). At each target $\varepsilon$ (expected batch $1024$, $100$ epochs, $C=1$, $\delta=n_{\mathrm{train}}^{-1.1}$): the pixel reconstruction bounds and, on MNIST and CIFAR-10, the perceptual bounds (left axis), each a $95\%$ lower bound holding against every adversary of \Cref{sec:setting} and every architecture at that configuration, and the test accuracy the same run attains (right axis, mean $\pm$ standard deviation over $3$ seeds; non-private baseline as the dotted horizontal line). Shaded: the $\varepsilon\le1$ (dark) and $1<\varepsilon\le10$ (light) regimes of \citet{ponomareva2023dp}; $\varepsilon>10$ is unshaded. }
    \label{fig:spectrum_all}
\end{figure}

\begin{table}[htbp]
    \centering\footnotesize\setlength{\tabcolsep}{3pt}
    \caption{Values behind \Cref{fig:spectrum_headline,fig:spectrum_all}: test accuracy (mean over $3$ seeds) and the pixel and perceptual bounds at each target $\varepsilon$, each a $95\%$ lower bound. $\dagger$: the $\varepsilon>10$ band of \citet{ponomareva2023dp}.}
    \label{tab:spectrum}
\begin{tabular}{r ccc ccc ccc cc}
\toprule
 & \multicolumn{3}{c}{MNIST} & \multicolumn{3}{c}{CIFAR-10} & \multicolumn{3}{c}{CelebA} & \multicolumn{2}{c}{ImageNette} \\
\cmidrule(lr){2-4}\cmidrule(lr){5-7}\cmidrule(lr){8-10}\cmidrule(lr){11-12}
$\varepsilon$ & acc. & $\rhostarpx$ & $\rhostarlp$ & acc. & $\rhostarpx$ & $\rhostarlp$ & acc. & $\rhostarpx$ & $\rhostarlp$ & acc. & $\rhostarpx$ \\
\midrule
$1$ & 0.313 & 0.907 & 0.908 & 0.187 & 0.887 & 0.952 & 0.923 & 0.888 & 0.944 & 0.122 & 0.865 \\
$2$ & 0.865 & 0.892 & 0.886 & 0.373 & 0.867 & 0.940 & 0.946 & 0.861 & 0.931 & 0.239 & 0.845 \\
$4$ & 0.953 & 0.845 & 0.831 & 0.564 & 0.815 & 0.901 & 0.954 & 0.805 & 0.892 & 0.440 & 0.795 \\
$8$ & 0.970 & 0.754 & 0.726 & 0.655 & 0.711 & 0.822 & 0.957 & 0.696 & 0.801 & 0.603 & 0.683 \\
$16^{\dagger}$ & 0.971 & 0.556 & 0.536 & 0.703 & 0.523 & 0.612 & 0.961 & 0.513 & 0.588 & 0.662 & 0.445 \\
$32^{\dagger}$ & 0.972 & 0.131 & 0.169 & 0.739 & 0.131 & 0.188 & 0.967 & 0.131 & 0.174 & 0.702 & 0.000 \\
$64^{\dagger}$ & 0.976 & 0.000 & 0.000 & 0.765 & 0.000 & 0.000 & 0.963 & 0.000 & 0.000 & 0.723 & 0.000 \\
\midrule
non-priv. & 0.989 & --- & --- & 0.905 & --- & --- & 0.971 & --- & --- & 0.779 & --- \\
\bottomrule
\end{tabular}
\end{table}

\begin{table}[htbp]
    \centering\small
\caption{Mechanism internals behind \Cref{fig:spectrum_headline,fig:spectrum_all}: the leakage each configuration incurs, the noise multiplier achieving it, and the binding pair. A dash: the bound is zero at that leakage; n/a: LPIPS is not evaluated on ImageNette. $k^\star$ grows with $\gamma$ overall up to single grid steps on CelebA.}
\label{tab:spectrum_full}
\begin{tabular}{l r r r cc cc}
\toprule
Dataset & $\varepsilon$ & $\gamma$ (bits) & $\sigma$ & \multicolumn{2}{c}{$k^\star$ (pix / $\lpipsmap$)} & \multicolumn{2}{c}{$\nu^{2\star}$ (pix / $\lpipsmap$)} \\
\midrule
MNIST & $1$ & 0.048 & 5.574 & 14 & 16 & 1.6 & 3.2 \\
 & $2$ & 0.169 & 3.028 & 16 & 32 & 1.6 & 6.4 \\
 & $4$ & 0.570 & 1.741 & 16 & 32 & 1.6 & 3.2 \\
 & $8$ & 1.785 & 1.110 & 28 & 32 & 1.6 & 1.6 \\
 & $16$ & 5.042 & 0.796 & 28 & 32 & 0.8 & 0.8 \\
 & $32$ & 13.150 & 0.617 & 28 & 32 & 0.4 & 0.4 \\
 & $64$ & 32.883 & 0.497 & -- & -- & -- & -- \\
\midrule
CIFAR-10 & $1$ & 0.049 & 5.838 & 6 & 16 & 0.8 & 1.6 \\
 & $2$ & 0.171 & 3.167 & 6 & 32 & 0.4 & 3.2 \\
 & $4$ & 0.577 & 1.815 & 12 & 32 & 0.8 & 3.2 \\
 & $8$ & 1.806 & 1.147 & 32 & 32 & 1.6 & 1.6 \\
 & $16$ & 5.108 & 0.816 & 32 & 32 & 0.8 & 0.8 \\
 & $32$ & 13.317 & 0.629 & 32 & 32 & 0.4 & 0.4 \\
 & $64$ & 33.222 & 0.504 & -- & -- & -- & -- \\
\midrule
CelebA & $1$ & 0.049 & 5.838 & 3 & 16 & 0.2 & 1.6 \\
 & $2$ & 0.171 & 3.167 & 6 & 32 & 0.4 & 3.2 \\
 & $4$ & 0.577 & 1.815 & 8 & 32 & 0.4 & 3.2 \\
 & $8$ & 1.806 & 1.147 & 16 & 32 & 0.8 & 1.6 \\
 & $16$ & 5.108 & 0.816 & 32 & 32 & 0.8 & 0.8 \\
 & $32$ & 13.317 & 0.629 & 24 & 32 & 0.2 & 0.4 \\
 & $64$ & 33.222 & 0.504 & -- & -- & -- & -- \\
\midrule
ImageNette & $1$ & 0.060 & 10.969 & 24 & n/a & 3.2 & n/a \\
 & $2$ & 0.207 & 5.941 & 24 & n/a & 3.2 & n/a \\
 & $4$ & 0.682 & 3.313 & 32 & n/a & 3.2 & n/a \\
 & $8$ & 2.111 & 1.948 & 32 & n/a & 1.6 & n/a \\
 & $16$ & 6.010 & 1.242 & 32 & n/a & 0.8 & n/a \\
 & $32$ & 15.682 & 0.863 & -- & -- & -- & -- \\
 & $64$ & 38.345 & 0.639 & -- & -- & -- & -- \\
\bottomrule
\end{tabular}
\end{table}

\subsection{The attack on DP-SGD}\label{app:exp_attack}\label{app:exp_attacks}\label{app:exp_multistep}

\paragraph{Model.} We attack the model of \citet{schwethelm2024visual}: a single linear layer whose scalar output is the loss, $\ell(W;u)=Wu$, so that the per-example gradient is the input $u$ itself, whatever the value of $W$. Nothing of the record is lost to the architecture, and since \Cref{thm:gc_canonical} holds for every model and loss at a given $(q_t,\sigma_t)$, it is the model on which an attack can come closest to the bound. The input is centred with a mean fixed on $D_{\mathrm{shadow}}$: $u=x-\bar x$ in pixel space and $u=\lpipsmap(x)-\bar\Psi$ in feature space, where the frozen feature map of \Cref{app:exp_phi} is placed before the layer so that the attack estimates the features directly. The informed adversary knows the other records and which of them were sampled, so it subtracts their clipped gradients from the released sum. What remains of step $t$ is
\begin{equation}\label{eq:witness_channel}
    Y_t \;=\; \mathbf 1_t\,a(u)\,u \;+\; \sigma C\,\xi_t,
\end{equation}
the target's clipped gradient $a(u)u$, with $a(u)=\min\{1,C/\|u\|\}$ the clip factor, present or absent according to the target's participation $\mathbf 1_t\sim\mathrm{Bernoulli}(q)$, plus the mechanism's noise $\sigma C\,\xi_t$ with $\xi_t\sim\mathcal N(0,I)$. The adversary does not observe $\mathbf 1_t$ ($\mathbf 1_t\equiv1$ at $q=1$). We take $C=0.3$, small enough that every gradient is clipped, as \citet{hayes2023bounding} do for the same reason (their App.~K.1).

\paragraph{Attack.} We use the prior-aware attack of \citet{hayes2023bounding}, in the likelihood form of their Appendix~G. The adversary holds candidates $z_1,\dots,z_n$ and, for each, computes the probability of the releases if the target were $z_j$. Under that hypothesis each release is either $c_j=a(z_j)z_j$ plus noise, if the target was sampled at that step, or noise alone, so the probability is a product over steps of two-way mixtures,
\begin{equation}\label{eq:candidate_loglik}
    p(Y_{1:T}\mid z_j) \;\propto\; \prod_{t=1}^T \Bigl[(1-q)+q\,e^{L_{tj}}\Bigr], \qquad L_{tj}=\frac{2\langle Y_t, c_j\rangle-\|c_j\|^2}{2\sigma^2C^2},
\end{equation}
which is their score $s_z$ with a general clip norm in place of their $C=1$. Participation is integrated out exactly, step by step, and the clip factor of every candidate is known, so the score is in closed form for any $q$. We depart from \citet{hayes2023bounding} in two ways. Their adversary is given a shortlist that contains the target (their Section~3.4); ours holds the shadow sample $D_{\mathrm{shadow}}$ ($58{,}000$ records on MNIST, $44{,}000$ on CelebA), which does not, so it knows the source only through a sample of it and is weaker than the adversary of \Cref{sec:setting}. And where they report the candidate with the largest score (their Appendix~G), we report the posterior mean
\begin{equation}\label{eq:candidate_mean}
    \hat u \;=\; \sum_{j=1}^n w_j\,z_j, \qquad w_j \;=\; \frac{p(Y_{1:T}\mid z_j)}{\sum_{k=1}^n p(Y_{1:T}\mid z_k)},
\end{equation}
the candidates averaged by how well each explains the releases, since the bound controls squared error and the posterior mean minimises it under the candidate prior. The reconstruction is $\bar x+\hat u$ in pixel space and $\bar\Psi+\hat u$ in feature space, and costs one matrix product per target.

\paragraph{Evaluation.} Both configurations, $(q,T)=(1,1)$ and $(0.02,100)$, are run on the same $100$ held-out targets (\Cref{app:dh_attack}), each attacked once, with $\sigma$ solved so that \Cref{thm:gc_canonical} charges the target $\gamma$. A marker is the ratio \eqref{eq:rho_def} with the expectation replaced by the mean over targets, divided by the $\sqrt{\tr\hat\Sigma}$ of \Cref{tab:trsigma}, the same scale as the bound; in the LPIPS panels it is a feature-space error, which is what the bound controls, and the dashed lines show, for both configurations, the error after inverting $\hat u$ to an image by gradient descent on $\|\lpipsmap(\hat x)-\hat z\|^2$. Two features of \Cref{fig:tightness,fig:tightness_mnist} are worth stating. Subsampling weakens the attack at matched $\gamma$ by up to $0.08$ in $\rho$ (paired differences $-0.03$ to $+0.08$; it is slightly stronger only on CelebA LPIPS at $72$ bits), the cost of not knowing when the target participated, as \citet{hayes2023bounding} also observe. And at $72$ bits the CelebA markers stop at $0.71$--$0.73$ (pixel) and $0.91$--$0.93$ (LPIPS), over the two configurations. The attack can only output a weighted average of shadow records, so it can get no closer to a held-out face than such an average can, and $D_{\mathrm{shadow}}$ holds no near duplicate of a held-out face. This limits the attack, not the bound, which depends on the source's entropy rather than on any sample of it, and is zero at that leakage in any case.

\begin{table}[htbp]
    \centering\small
    \caption{The attack against the bound: $\rho$ over the $100$ targets, bootstrap standard errors between $0.013$ and $0.024$.}
    \label{tab:attack_results}
    \begin{tabular}{l l c c c c}
        \toprule
        & $\gamma$ (bits) & $0.72$ & $2.89$ & $8.01$ & $72.13$ \\
        \midrule
        MNIST, pixel  & $q=1$               & $0.987$ & $0.948$ & $0.848$ & $0.620$ \\
                      & $(q,T)=(0.02,100)$  & $1.006$ & $0.984$ & $0.921$ & $0.690$ \\
                      & bound               & $0.831$ & $0.673$ & $0.397$ & $0$ \\
        \midrule
        MNIST, LPIPS  & $q=1$               & $0.999$ & $0.922$ & $0.845$ & $0.636$ \\
                      & $(q,T)=(0.02,100)$  & $1.001$ & $0.964$ & $0.926$ & $0.681$ \\
                      & bound               & $0.816$ & $0.653$ & $0.396$ & $0$ \\
        \midrule
        CelebA, pixel & $q=1$               & $0.920$ & $0.884$ & $0.768$ & $0.707$ \\
                      & $(q,T)=(0.02,100)$  & $0.970$ & $0.941$ & $0.840$ & $0.725$ \\
                      & bound               & $0.789$ & $0.628$ & $0.373$ & $0$ \\
        \midrule
        CelebA, LPIPS & $q=1$               & $0.983$ & $0.963$ & $0.911$ & $0.933$ \\
                      & $(q,T)=(0.02,100)$  & $0.991$ & $0.992$ & $0.975$ & $0.906$ \\
                      & bound               & $0.878$ & $0.734$ & $0.459$ & $0$ \\
        \bottomrule
    \end{tabular}
\end{table}

\begin{figure}[htbp]
    \centering
    \includegraphics[width=\textwidth]{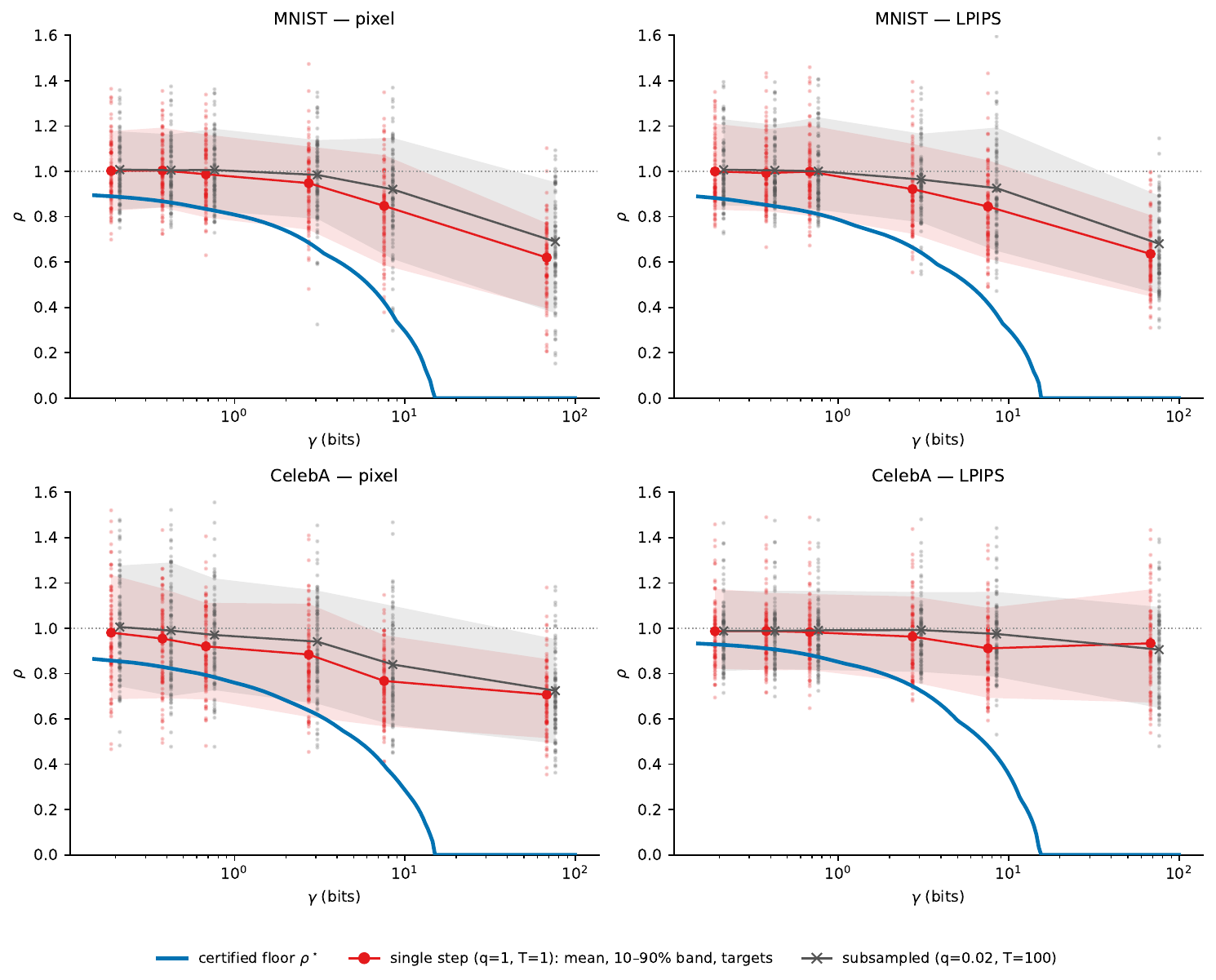}
    \caption{Per-target spread behind \Cref{fig:tightness,fig:tightness_mnist}, for the single-step (circles) and subsampled (crosses) attack. Lines: the mean over the $100$ targets; shaded: the $10$--$90$th percentile band of the per-target error at each leakage; faint points: every target. The bound is on expected error. At low leakage a sizeable fraction of individual targets fall below it ($9$--$40\%$ at $\gamma\le0.72$ bits, $1$--$11\%$ at $2.89$ bits, at most $3\%$ at $8.01$ bits), which the bound permits.}
    \label{fig:sandwich_full}
\end{figure}

\subsection{Tightness at fixed leakage}\label{app:exp_compressor}

The bound \eqref{eq:rho_certified} sees the mechanism only through $\gamma$, so it holds for every release that leaks at most $\gamma$ bits about the record, whether or not that release is a model. It can therefore be tight at a given $\gamma$ only against the release whose bits buy the most reconstruction, which is not DP-SGD. We construct one that is close to it, a vector quantiser, and use it to locate the slack of the bound. At each leakage we compute
\begin{equation}\label{eq:sandwich_chain}
\rho^\star \;\le\; \underbrace{\rho_{\mathrm{blk}}}_{\text{\Cref{prop:block_leakage}}} \;\le\; \underbrace{\rho_{\mathrm{cb}}}_{\text{attack on the quantiser}} \;\le\; \underbrace{\rho_{\mathrm{dpsgd}}}_{\text{attack on DP-SGD}} ,
\end{equation}
where $\rho_{\mathrm{blk}}$ relaxes the bound while remaining a lower bound on every attack of the quantiser (\Cref{tab:chain}). Across the four panels of \Cref{fig:tightness,fig:tightness_mnist} the first gap is $0.10$--$0.33$ in $\rho$ and the second at most $0.03$.

\paragraph{The codebook mechanism.} The distortion--rate function is the least expected squared error that a release of $R$ bits permits. It is approached by quantisers, which release the index of the codeword nearest to the record and so spend every released bit on reducing squared error \citep[Ch.~10]{cover2006elements}. We use a single-record quantiser of this kind. No model is trained: the mechanism releases the index $Q\in\{1,\dots,M\}$ of the codeword nearest to $\Psi(x_i)$, among $M=2^R$ codewords in the space of $\Psi$. The codewords are built in two steps. First, $k$-means on $D_{\mathrm{shadow}}$ gives $M$ centroids. Second, we assign $D_{\mathrm{est}}$ to these centroids with at most $\lceil n_{\mathrm{est}}/M\rceil$ records per cell (a single balanced step in the spirit of \citet{malinen2014balanced}), and take the cell means as the codewords. The size cap spreads the codewords so that every cell stays populated, which estimating the per-block leakage below requires. The codebook is fixed independently of the held-out targets and is known to the adversary; for LPIPS the same construction is applied to $\lpipsmap(x)$. Every bound below holds for any codebook, and a worse codebook only widens the measured gaps, so its optimality is not at issue. Since $Q$ is a deterministic function of $x_i$ with $M$ values, independent of $X^{-i}$, the release leaks $I(x_i;Q\mid X^{-i})=H(Q)\le R$ bits. We charge $\gamma=R$, which slightly overstates the leakage (the realised $H(Q)$ is within $0.07$ bits of $R$) and so only lowers $\rho^\star$. Below one bit, which a deterministic codebook cannot reach, we take the two-cell codebook and pass $Q$ through a binary symmetric channel: the release $Y$ equals $Q$ with probability $p\in[\tfrac12,1]$ and the other index otherwise, independently of everything else. Because $Y$ depends on $x_i$ only through $Q$, it leaks exactly $\gamma=I(Q;Y)=H_b\bigl(f_0p+f_1(1-p)\bigr)-H_b(p)$ bits \citep[Ch.~7]{cover2006elements}, where $f_0,f_1$ are the cell frequencies, and every $\gamma\in(0,H_b(f_0)]$ is reached by choice of $p$.

\paragraph{The codebook attack.} The adversary sees the release $Z$ ($Z=Q$ for the deterministic codebook, $Z=Y$ for the randomised one) and the mechanism does not involve $X^{-i}$, so its reconstruction is a function $g(Z)$. Under squared error the optimal $g$ is the conditional mean $\E[\Psi(x_i)\mid Z]$ \citep{bishop2006pattern}. For the deterministic codebook this is the mean of the released cell, and for the randomised codebook it is the average of the two cell means, each weighted by the probability that it is the true cell given the release. We estimate the cell means and frequencies on $D_{\mathrm{est}}$ and report the error on targets held out from $D_{\mathrm{shadow}}$ and $D_{\mathrm{est}}$. The estimate is a function of the release and of records disjoint from the target, so it is a valid reconstruction of \Cref{sec:setting}, and estimation error can only raise its error above the optimum, which widens rather than narrows the measured gap.

\paragraph{The bounds.} In the proof of \Cref{cor:rp} the mechanism enters in exactly one place: the rate at which each block's Shannon lower bound is taken, every block being charged the whole of $\gamma$. Write $\tilde z_b$ for the dithered block $\tilde z^{\,\Pi_b\Psi}_i$ and $\gamma_b := I(\tilde z_b;\theta\mid X^{-i})$ for the leakage block $b$ receives, which for the quantiser is $I(\tilde z_b;Q)=h(\tilde z_b)-\sum_c\Pr(Q=c)\,h(\tilde z_b\mid Q=c)$. Both bounds in \eqref{eq:sandwich_chain} have the form
\begin{equation}\label{eq:floor_form}
\rho^2 \;\ge\; \frac{k}{2\pi e\,\tr\Sigma_\Psi}\cdot\frac{1}{N_k}\sum_{b=1}^{N_k}F(b)\;-\;\nu^2 ,
\end{equation}
with $F_{\mathrm{cert}}=2^{\,2(h(\tilde z_b)-\gamma)/k}$ the bound before its Jensen step and $F_{\mathrm{blk}}=2^{\,2(h(\tilde z_b)-\gamma_b)/k}$ charging each block only the leakage it received. The difference is the \emph{per-block charge}.

\begin{table}[htbp]
\centering\footnotesize
\setlength{\tabcolsep}{3pt}
\newcolumntype{Y}{>{\raggedright\arraybackslash}X}
\caption{The chain \eqref{eq:sandwich_chain}: what each term is, how it is evaluated, and what the gap between consecutive terms measures.}
\label{tab:chain}
\begin{tabularx}{\textwidth}{@{}l l Y Y Y@{}}
\toprule
& $F(b)$ in \eqref{eq:floor_form} & Lower-bounds & How evaluated & Gap to the row above \\
\midrule \midrule
$\rho^\star$ & $F_{\mathrm{cert}}=2^{\,2(h(\tilde z_b)-\gamma)/k}$ & every attack on every release leaking $\le\gamma$ & \Cref{prop:coverage} at $\gamma=R$, with its grid and margins & --- \\ \midrule
$\rho_{\mathrm{blk}}$ & $F_{\mathrm{blk}}=2^{\,2(h(\tilde z_b)-\gamma_b)/k}$ & every attack on this release & $k=1$, $\nu^2=0.2$, one-dimensional plug-ins \eqref{eq:mc_def} on $D_{\mathrm{est}}$ for $h(\tilde z_b)$ and each $h(\tilde z_b\mid Q=c)$; point estimate; every cell must hold $20$ records, so $R\le8$ (pixel) and $R\le7$ (LPIPS) & per-block charge \\ \midrule
$\rho_{\mathrm{cb}}$ & --- & --- & cell-mean decoder; error on held-out targets, bootstrap standard error & residual slack of the Shannon lower bound at rate $\gamma_b$, plus sampling error of the targets \\ \midrule
$\rho_{\mathrm{dpsgd}}$ & --- & --- & posterior mean under the empirical shadow prior (\Cref{app:exp_attack}) on one full-batch step, at $\gamma=1/(2\ln2\,\sigma^2)$ bits, and on $T=100$ steps at $q=0.02$ & what DP-SGD fails to convert into reconstruction \\
\bottomrule
\end{tabularx}
\end{table}

\begin{proposition}[Per-block bound]\label{prop:block_leakage}
In the setting of \Cref{cor:rp}, $\gamma_b\le\gamma_i$ for every $b$, and \eqref{eq:floor_form} holds with $F_{\mathrm{blk}}$. In particular $F_{\mathrm{cert}}\le F_{\mathrm{blk}}$ at any common $(k,\nu^2)$.
\end{proposition}
\begin{proof}
$\tilde z_b\to x_i\to(\theta,X^{-i})$ is a Markov chain (\Cref{def:dither}), so $\gamma_b\le\gamma_i$ by data processing, and $F_{\mathrm{cert}}\le F_{\mathrm{blk}}$ follows. The proof of \Cref{thm:master} at the representation $\Pi_b\circ\Psi$ uses the leakage only in Step~1, as $I(\tilde z_b;\Pi_b\zhatT_i)\le\gamma_i$, and data processing along $\tilde z_b\to(\theta,X^{-i})\to\Pi_b\zhatT_i$ gives the sharper $I(\tilde z_b;\Pi_b\zhatT_i)\le\gamma_b$. Steps~2--4 then give $\E\|\Pi_b(\Psi(x_i)-\zhatT_i)\|^2\ge\frac{k}{2\pi e}2^{2(h(\tilde z_b)-\gamma_b)/k}-ks^2$ for every block, and averaging over $b$ by \eqref{eq:block_isometry} and dividing by $\tr\Sigma_\Psi$ gives \eqref{eq:floor_form} with $F_{\mathrm{blk}}$.
\end{proof}

\paragraph{Why the relaxed bounds use $k=1$.} The per-block bound is evaluated at $k=1$. Its exponent carries $\gamma_b$ rather than $\gamma$, so nothing pushes it towards large $k$, and one-dimensional plug-ins suffice for $\gamma_b=h(\tilde z_b)-\sum_c\Pr(Q=c)\,h(\tilde z_b\mid Q=c)$, which is why every cell must hold enough records.

\paragraph{Interpretation.} The chain splits the distance from the bound to the DP-SGD attack into three gaps. The first, from $\rho^\star$ to $\rho_{\mathrm{blk}}$, is $0.10$--$0.33$ across the four panels of \Cref{fig:tightness,fig:tightness_mnist}. It is the cost of charging every block the full $\gamma$. Removing it would require either a per-block leakage bound for the mechanism or a lower bound on the entropy of the whole record, and the plug-in caps the latter at $\log_2 n$ bits above the dither entropy (\Cref{app:why}). The second, from $\rho_{\mathrm{blk}}$ to $\rho_{\mathrm{cb}}$, is at most $0.03$ (\Cref{tab:tightness}, last column), and slightly negative, within standard error, on CelebA pixels at $R=3$--$5$. Since $\rho_{\mathrm{blk}}$ lower-bounds every attack on the codebook, this small gap means both that the Shannon lower bound is close to tight for these sources and that our codebook attack is close to optimal. The third, from $\rho_{\mathrm{cb}}$ to $\rho_{\mathrm{dpsgd}}$ (\Cref{tab:tightness,tab:attack_results}), is specific to DP-SGD. The leakage $\gamma$ fixes how many bits are released, not which: the codebook spends its bits on a cell index chosen for reconstruction, whereas DP-SGD spends them on a noisy clipped gradient spread over every direction, so the same $\gamma$ buys far less reconstruction. This gap is the margin by which \Cref{fig:spectrum_headline,fig:spectrum_all} are conservative for DP-SGD. It may also include slack in the accountant, since $\gamma$ upper-bounds DP-SGD's true leakage, and the suboptimality of our attack, neither of which the chain can separate from the mechanism.

\begin{figure}[htbp]
    \centering
    \includegraphics[width=\textwidth]{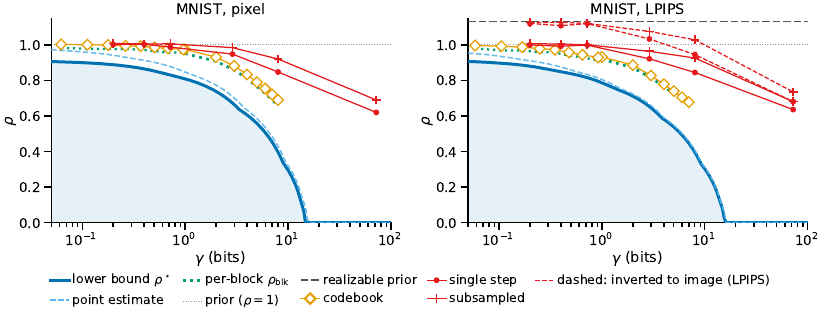}
    \caption{Tightness at fixed leakage on MNIST, in pixel (left) and LPIPS (right) error, as in \Cref{fig:tightness}: the chain \eqref{eq:sandwich_chain}, with terms as in \Cref{tab:chain}. No attack crosses the bound. $\rho_{\mathrm{blk}}$ is drawn while every codebook cell holds $20$ estimation records. Dotted: the prior at $\rho=1$. On MNIST the attack reaches lower $\rho$ than on CelebA at high leakage because the shadow sample contains near-duplicates of many held-out digits (\Cref{fig:exemplars,fig:exemplars_celeba}).}
    \label{fig:tightness_mnist}
\end{figure}

\begin{table}[htbp]
    \centering\small
    \setlength{\tabcolsep}{5pt}
    \caption{Values behind \Cref{fig:tightness,fig:tightness_mnist}, one block per panel. Gaps are computed before rounding.}
    \label{tab:tightness}
    \begin{tabular}{r rr r rr}
        \toprule
        $R$ & $\rho^\star$ & $\rho_{\mathrm{blk}}$ & $\rho_{\mathrm{cb}}$ (SE) & $\rho_{\mathrm{blk}}-\rho^\star$ & $\rho_{\mathrm{cb}}-\rho_{\mathrm{blk}}$ \\
        \midrule
        \multicolumn{6}{l}{\emph{MNIST, pixel}} \\
        $1$ & $0.809$ & $0.951$ & $0.975$ ($0.005$) & $0.142$ & $0.023$ \\
        $2$ & $0.739$ & $0.908$ & $0.931$ ($0.005$) & $0.169$ & $0.023$ \\
        $3$ & $0.665$ & $0.859$ & $0.882$ ($0.005$) & $0.194$ & $0.023$ \\
        $4$ & $0.606$ & $0.811$ & $0.834$ ($0.005$) & $0.205$ & $0.023$ \\
        $5$ & $0.558$ & $0.765$ & $0.790$ ($0.005$) & $0.207$ & $0.026$ \\
        $6$ & $0.507$ & $0.728$ & $0.754$ ($0.005$) & $0.220$ & $0.027$ \\
        $7$ & $0.454$ & $0.692$ & $0.721$ ($0.005$) & $0.238$ & $0.029$ \\
        $8$ & $0.397$ & $0.658$ & $0.691$ ($0.005$) & $0.261$ & $0.032$ \\
        \midrule
        \multicolumn{6}{l}{\emph{CelebA, pixel}} \\
        $1$ & $0.760$ & $0.858$ & $0.862$ ($0.006$) & $0.098$ & $0.004$ \\
        $2$ & $0.682$ & $0.809$ & $0.810$ ($0.005$) & $0.127$ & $0.001$ \\
        $3$ & $0.622$ & $0.772$ & $0.770$ ($0.005$) & $0.149$ & $-0.002$ \\
        $4$ & $0.566$ & $0.737$ & $0.737$ ($0.005$) & $0.171$ & $0.000$ \\
        $5$ & $0.518$ & $0.711$ & $0.709$ ($0.005$) & $0.193$ & $-0.002$ \\
        $6$ & $0.472$ & $0.687$ & $0.687$ ($0.005$) & $0.215$ & $0.000$ \\
        $7$ & $0.424$ & $0.662$ & $0.667$ ($0.005$) & $0.239$ & $0.005$ \\
        $8$ & $0.374$ & $0.641$ & $0.651$ ($0.004$) & $0.267$ & $0.010$ \\
        \midrule
        \multicolumn{6}{l}{\emph{MNIST, LPIPS}} \\
        $1$ & $0.787$ & $0.920$ & $0.932$ ($0.004$) & $0.133$ & $0.012$ \\
        $2$ & $0.712$ & $0.871$ & $0.886$ ($0.004$) & $0.159$ & $0.015$ \\
        $3$ & $0.646$ & $0.815$ & $0.828$ ($0.005$) & $0.169$ & $0.013$ \\
        $4$ & $0.581$ & $0.765$ & $0.778$ ($0.005$) & $0.184$ & $0.013$ \\
        $5$ & $0.538$ & $0.725$ & $0.740$ ($0.005$) & $0.186$ & $0.015$ \\
        $6$ & $0.493$ & $0.688$ & $0.706$ ($0.004$) & $0.195$ & $0.017$ \\
        $7$ & $0.446$ & $0.657$ & $0.678$ ($0.004$) & $0.211$ & $0.021$ \\
        \midrule
        \multicolumn{6}{l}{\emph{CelebA, LPIPS}} \\
        $1$ & $0.851$ & $0.962$ & $0.980$ ($0.005$) & $0.112$ & $0.017$ \\
        $2$ & $0.789$ & $0.939$ & $0.955$ ($0.005$) & $0.150$ & $0.016$ \\
        $3$ & $0.727$ & $0.917$ & $0.933$ ($0.005$) & $0.190$ & $0.016$ \\
        $4$ & $0.662$ & $0.898$ & $0.912$ ($0.005$) & $0.235$ & $0.015$ \\
        $5$ & $0.593$ & $0.878$ & $0.895$ ($0.005$) & $0.285$ & $0.017$ \\
        $6$ & $0.550$ & $0.857$ & $0.878$ ($0.005$) & $0.307$ & $0.020$ \\
        $7$ & $0.506$ & $0.839$ & $0.862$ ($0.005$) & $0.333$ & $0.023$ \\
        \bottomrule
    \end{tabular}
\end{table}

\FloatBarrier
\subsection{Additional figures}\label{app:exp_extra}

\begin{figure}[htbp]
    \centering
    \includegraphics[width=0.8\textwidth]{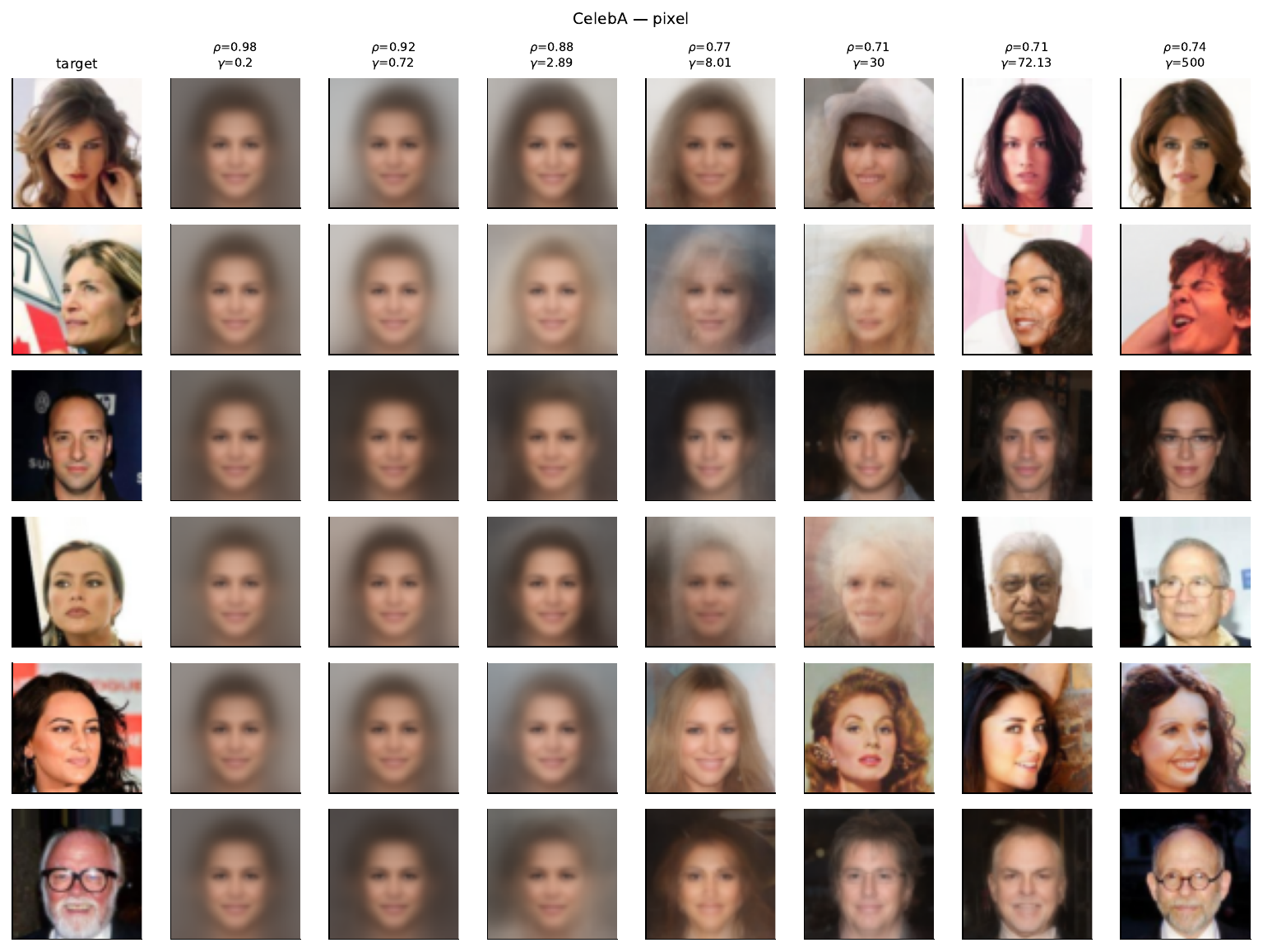}\\[6pt]
    \includegraphics[width=0.8\textwidth]{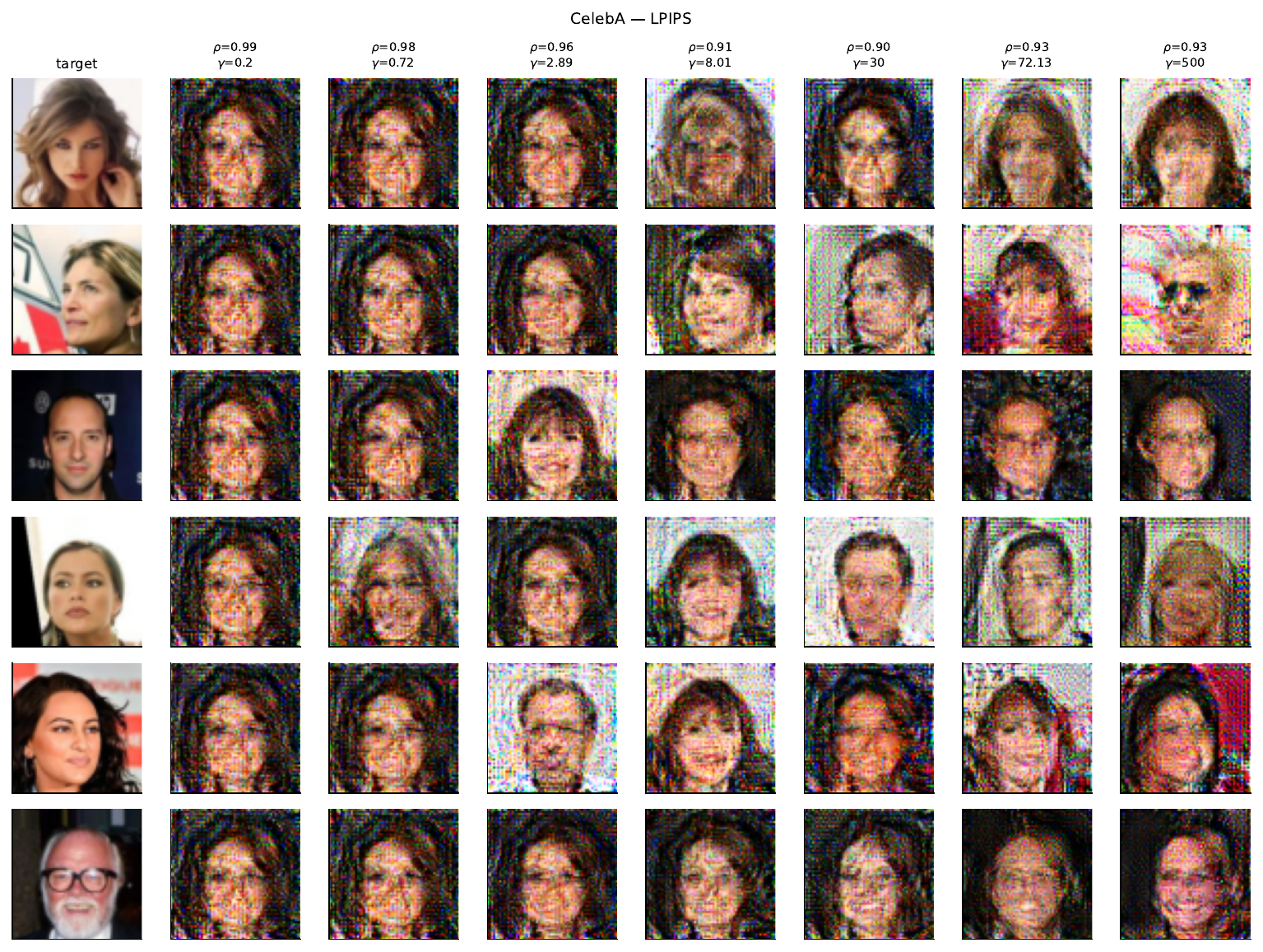}
    \caption{\textbf{Attack reconstructions at each leakage, CelebA.} As \Cref{fig:exemplars}, for CelebA in pixel space (top) and LPIPS space (bottom). $D_{\mathrm{shadow}}$ holds no near duplicate of a held-out face, so the pixel $\rho$ levels off near $0.71$ beyond $\gamma\approx70$ bits.}
    \label{fig:exemplars_celeba}
\end{figure}

\begin{figure}[htbp]
    \centering
    \includegraphics[width=0.8\textwidth]{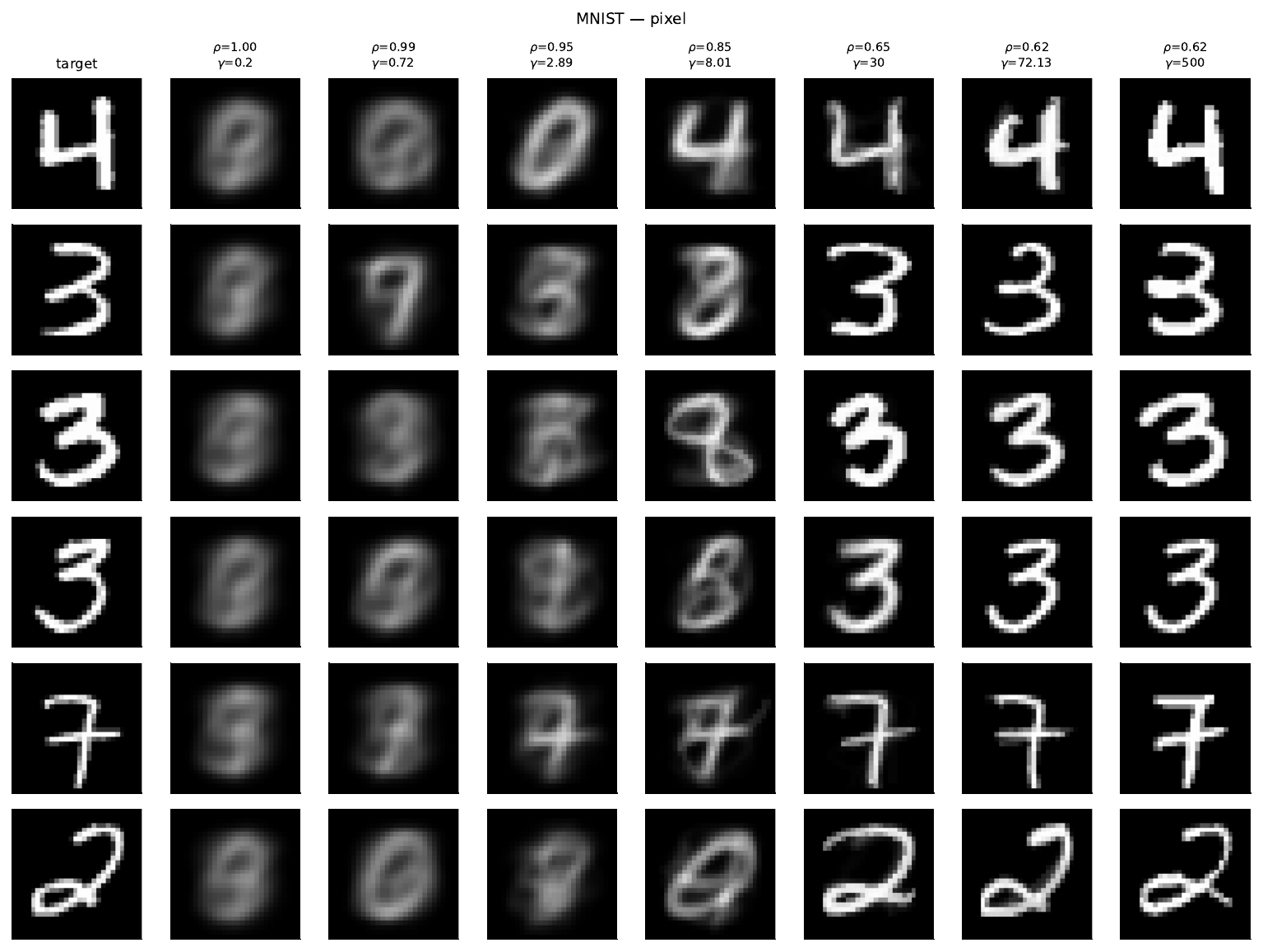}\\[6pt]
    \includegraphics[width=0.8\textwidth]{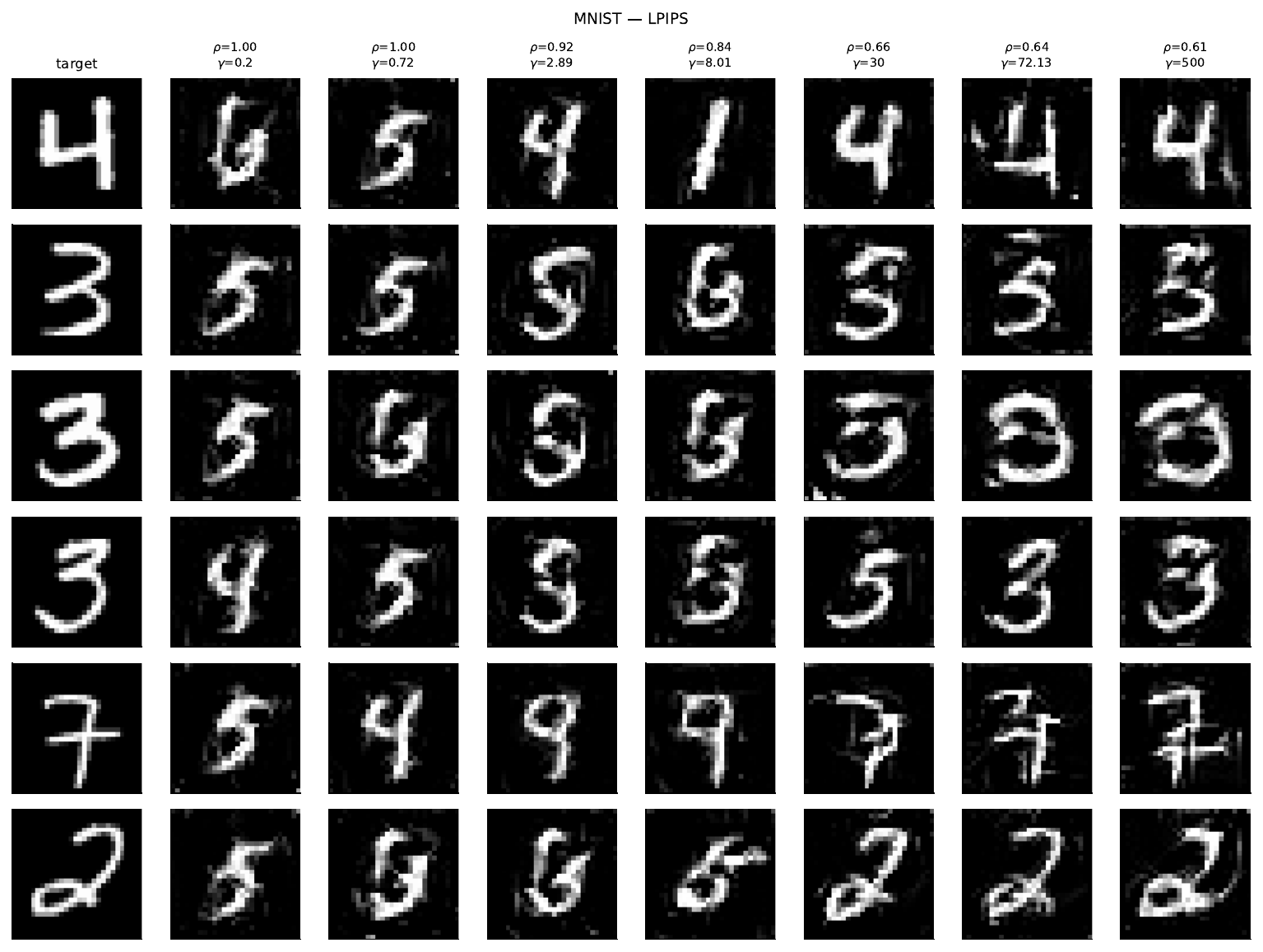}
    \caption{\textbf{Attack reconstructions at each leakage, MNIST.} Outputs of the single-step attack of \Cref{app:exp_attack} on $6$ fixed held-out targets at $7$ leakages, $\gamma\in\{0.2,0.72,2.89,8.01,30,72.13,500\}$ bits, in pixel space (top) and LPIPS space (bottom). Each column shows what reconstruction at its value of $\rho$ looks like in practice, so the bounds of \Cref{fig:spectrum_headline,fig:spectrum_all} can be read as images rather than numbers. In each panel the first column is the target and each further column is one leakage, headed by $\rho$, the mean over all $100$ targets at that leakage. The LPIPS panel shows the feature estimate inverted to an image, headed by the feature-space $\rho$. In pixel space most reconstructed digits are wrong at $\rho=0.85$ ($\gamma=8$ bits) and all six are right at $\rho=0.65$ ($\gamma=30$ bits). The attack outputs a weighted mix of shadow images, none of them close to a held-out target, so $\rho$ levels off near $0.62$ beyond $\gamma\approx30$ bits.}
    \label{fig:exemplars}
\end{figure}

\end{document}

%% file: math_commands.tex
\usepackage{amsmath,amsfonts,bm}

\def\eqref#1{equation~\ref{#1}}

\def\1{\bm{1}}

\DeclareMathAlphabet{\mathsfit}{\encodingdefault}{\sfdefault}{m}{sl}
\SetMathAlphabet{\mathsfit}{bold}{\encodingdefault}{\sfdefault}{bx}{n}

\newcommand{\E}{\mathbb{E}}

\newcommand{\R}{\mathbb{R}}

\newcommand{\KL}{D_{\mathrm{KL}}}
\newcommand{\Var}{\mathrm{Var}}

\newcommand{\Cov}{\mathrm{Cov}}